\documentclass[11pt, a4paper, confidential, copyright, goog]{google}

\usepackage[authoryear, sort&compress, round]{natbib}
\usepackage{cleveref}
\usepackage{booktabs} 
\usepackage{graphicx}
\usepackage{mathtools}
\usepackage{subcaption}
\usepackage{wrapfig}
\usepackage{comment}
\usepackage{algorithm}
\usepackage{algorithmic}

\uselogo{}

\newtheorem{theorem}{Theorem}

\newtheorem{lemma}{Lemma}

\def\mbf{\mathbf}
\def\mbb{\mathbb}
\def\mc{\mathcal}

\title{CoDistill-GRPO: A Co-Distillation Recipe for Efficient Group Relative Policy Optimization}

\correspondingauthor{kwonsm@umich.edu}

\reportnumber{} 

\author[1,*]{Soo Min Kwon}
\author[2]{Ziteng Sun}
\author[2]{Ananda Theertha Suresh}
\author[3]{Himanshu Jain}
\author[3]{Sanjiv Kumar}

\affil[1]{University of Michigan, Ann Arbor}
\affil[2]{Google Research}
\affil[3]{Google DeepMind}
\affil[*]{Work done during an internship at Google}

\begin{abstract}
Group Relative Policy Optimization (GRPO) has emerged as a powerful algorithm for improving the reasoning capabilities of language models, but often fails to improve small models due to sparse rewards on difficult tasks. Existing works mitigate this issue by leveraging a larger model, either to provide hints for rollouts or to provide dense reward signals through knowledge distillation (KD). However, this assumes the existence of such an oracle, and training one can significantly increase total training time.
In this work, we propose CoDistill-GRPO, a co-distillation algorithm that simultaneously trains a large and a small model by maximizing carefully designed GRPO objectives. The two models learn from each other: the small model uses an on-policy KD reward to learn from the large model's distribution, while the large model is updated using rollouts generated by the small model with importance reweighting, reducing the computational overhead of rollout generation. We show that CoDistill-GRPO substantially improves small model performance over standard GRPO on mathematical benchmarks across both Qwen and Llama models. Specifically, with Qwen2.5-Math-1.5B, we observe an accuracy increase of over 11.6 percentage points over the base model and an additional 6.0 percentage points over GRPO on the Minerva dataset. Interestingly, the larger model (Qwen2.5-Math-7B) trained with CoDistill-GRPO nearly matches standard GRPO performance despite training on small-model rollouts. This highlights CoDistill-GRPO as a cost-effective alternative to GRPO for larger models, yielding an approximate 18\% speedup, which may be of independent interest.

\end{abstract}

\begin{document}

\maketitle

\section{Introduction}

Reinforcement learning with verifiable rewards (RLVR)~\citep{cobbe2021trainingverifierssolvemath, wen2026reinforcement,lambert2025tulu, yue2025does, deepseekai2025deepseekr1incentivizingreasoningcapability} has emerged as a canonical paradigm for training large language models (LLMs) to efficiently solve complex reasoning tasks. The objective of RLVR is to maximize rewards derived from rule-based or deterministic verifiers (e.g., correct or incorrect binary rewards), and can generally be expressed in the following mathematical form:
\begin{align}\label{eqn:rlvr}
 \underset{\theta}{\mathrm{max}} \,\, &\mbb{E}_{q \sim P(Q)}  \Big[ \underbrace{\mbb{E}_{o \sim \pi_{\theta}(\cdot | q)} \left[r(q, o)\right]}_{\text{reward maximization}} - \beta \cdot  \underbrace{\mbb{D}_{\text{KL}}[\pi_{\theta}(\cdot | q) \| \pi_{\text{ref}}(\cdot | q)}_{\text{KL regularization}} ]  \Big],
\end{align}
where $o \sim \pi_{\theta}$ denotes an output sampled from some policy $\pi_{\theta}$ parameterized by $\theta$ for some query $q \sim P(Q)$, and $\beta \geq 0$ denotes the regularization parameter such that the policy $\pi_{\theta}$ does not deviate too much from some reference policy $\pi_{\text{ref}}$.

Since the work of DeepSeekMath~\citep{deepseek-math}, Group Relative Policy Optimization (GRPO) has received widespread attention, particularly for mathematical reasoning tasks. GRPO is particularly appealing for its simple yet effective objective, as it does not require training a separate value model for advantage computation, as is done in algorithms such as Proximal Policy Optimization (PPO)~\citep{ouyang2022training}. This sparked a wide range of efforts aimed at improving GRPO performance and enhancing its training stability~\citep{bnpo, liu2025understanding, gspo, dapo, shrivastava2025samplethinklessgroup, xu2025rolloutsusefuldownsamplingrollouts}.

On the other hand, recent research has identified the limitations of standard GRPO to effectively post-train small language models~\citep{zhang2025bread, kdrl}. The main challenge is the  ``learning cliff''~\citep{zhang2026scafgrpo}: on difficult reasoning tasks, GRPO on smaller models yields sparse rewards due to their limited capacity, leading to slow or stagnant learning compared to larger models (see \Cref{fig:frac_zeros}).
To address this issue, most existing works rely on these larger, more powerful models (or ground truth solutions) to either provide hints for rollouts~\citep{zhang2025bread} or to provide dense rewards through knowledge distillation (KD)~\citep{kdrl, agarwal2024onpolicy, lu2025onpolicydistillation}.
However, these methods typically assume the existence of a frozen large-model oracle throughout the training process. While one could first train a larger model and subsequently apply these techniques, such a sequential pipeline substantially increases total training time. Moreover, training a larger model first could widen the discrepancy between the small and large models, which has been shown to lead to poor learning when teacher and student models differ greatly in capacity and performance~\citep{huang2022knowledge, Cho2019OnTE, Mirzadeh2019ImprovedKD}.

\begin{figure}[t!]
    \centering
    \includegraphics[width=0.4\linewidth]{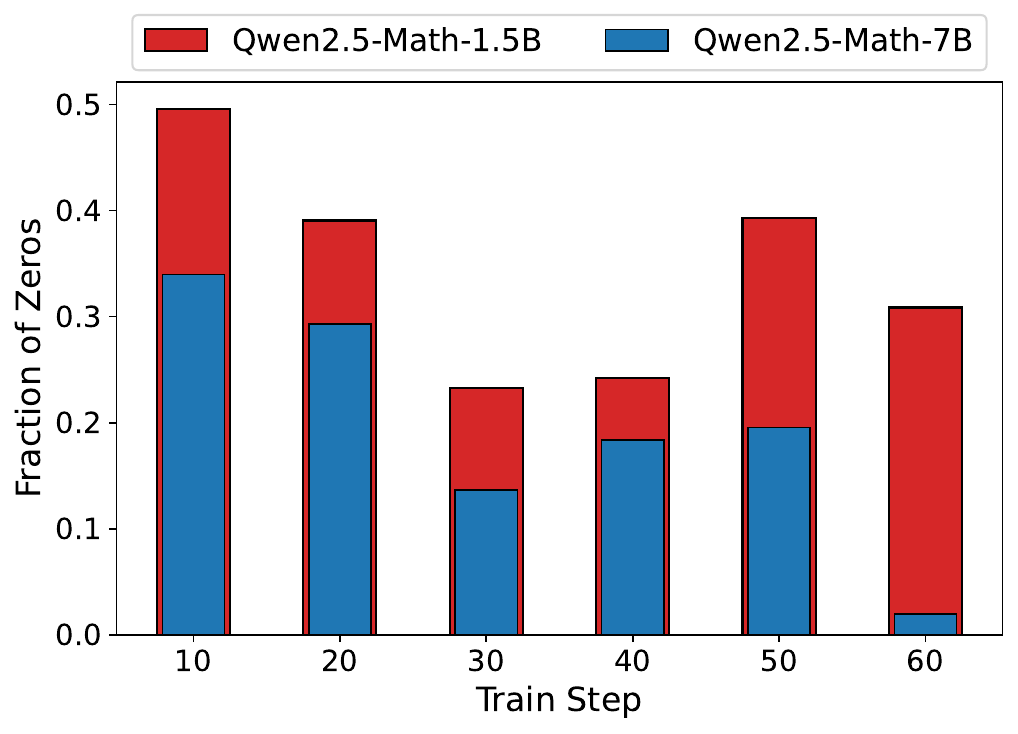}
    \caption{Fraction of zero accuracy rewards on the training iterations on the MATH dataset~\citep{hendrycks2021measuring} for the Qwen2.5-Math-1.5B and Qwen2.5-Math-7B models. The smaller model has a higher fraction of zero rewards, and hence learns slower across training iterations.
    }
    \label{fig:frac_zeros}
\end{figure}

In this work, we mitigate these issues with CoDistill-GRPO, a co-distillation recipe that simultaneously trains a large and a small model by maximizing carefully designed GRPO objectives. The two models learn from each other: the small model uses an on-policy KD reward to learn from the large model's distribution, while the large model is updated using rollouts generated solely by the small model with importance reweighting. 
Switching the generation phase to use only the smaller model significantly reduces training overhead, since rollout generation is the primary bottleneck in GRPO~\citep{liu2026specrlacceleratingonpolicyreinforcement, zheng2025act}. Interestingly,
we show that the on-policy KD reward is largely sufficient for the small model to produce outputs suitable for updating both models. 
To further improve the quality and diversity of small model rollouts, we take advantage of the fact that small model rollouts are inexpensive and increase the number of small model generations. Following \cite{xu2025rolloutsusefuldownsamplingrollouts}, the set of generations is then downsampled based on a combined accuracy and KD score to obtain enhanced relative advantage signals.

Across both Qwen and Llama model families, we demonstrate that CoDistill-GRPO substantially improves small model performance on mathematical reasoning benchmarks, including Minerva, MATH500, AMC2024, and OlympiadBench. For example, we show that Qwen2.5-Math-1.5B can achieve an accuracy of 32\% on the Minerva dataset, a 11.6 percentage point increase over the base model and over 6 percentage points improvement over GRPO. We also theoretically show that CoDistill-GRPO enjoys several statistical guarantees, such as unbiasedness of the gradient estimate.
Interestingly, the large model produced by CoDistill-GRPO nearly matches standard GRPO performance despite training on small model rollouts, achieving an 18\% improvement in training speed — highlighting that CoDistill-GRPO can potentially be used as a cost-effective alternative for  larger models.

\paragraph{Organization.} In \Cref{sec:background}, we provide a brief background on GRPO, along with a list of relevant work. In \Cref{sec:method}, we describe each step of CoDistill-GRPO, along with a justification of our design choices. Finally, in \Cref{sec:experiments} we provide experimental results showcasing CoDistill-GRPO, and in \Cref{sec:conclusion}, we conclude with a summary and directions for future work.

\begin{figure*}[t!]
    \centering
    \includegraphics[width=0.975\linewidth]{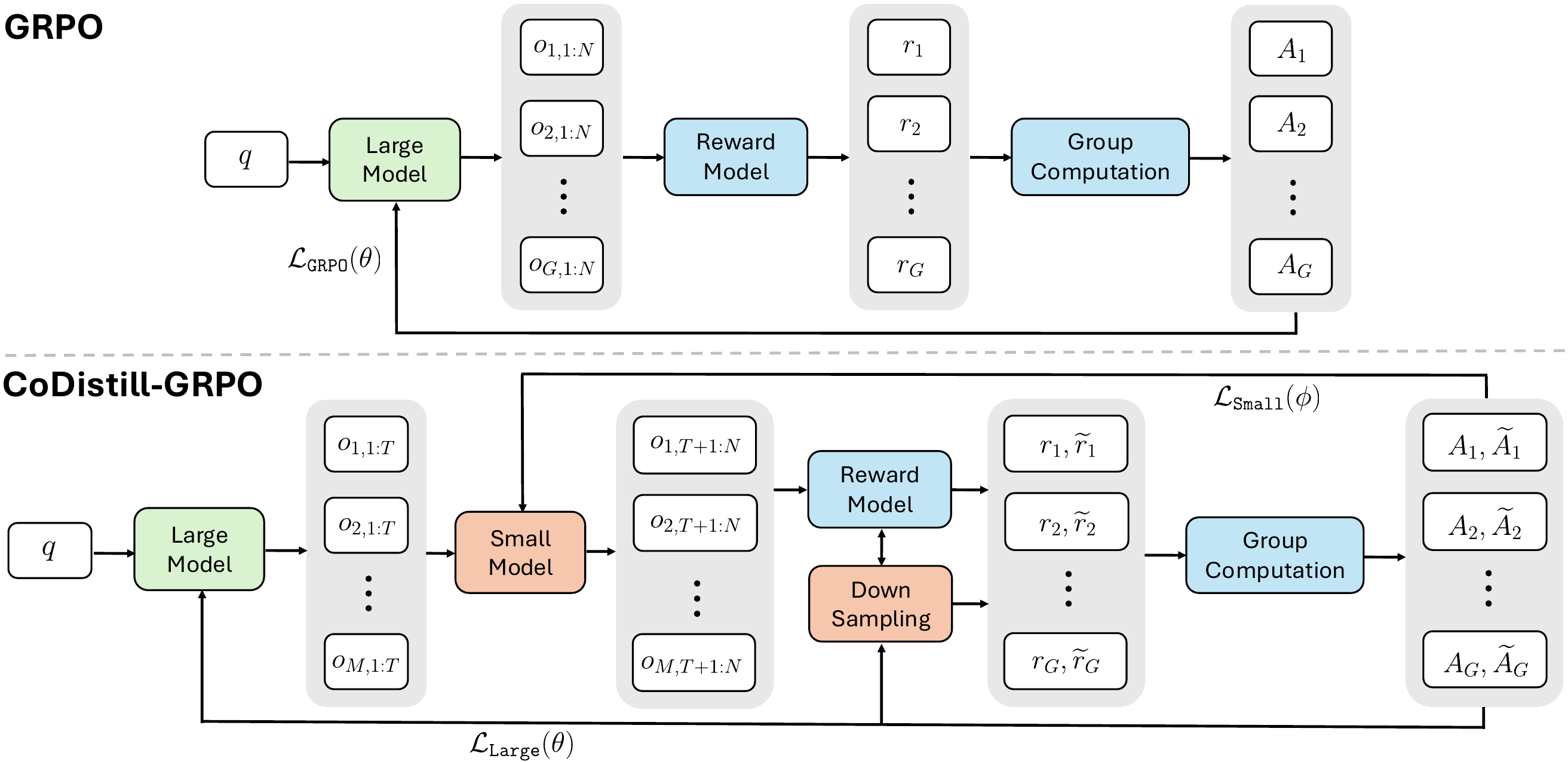}
    \caption{Flowchart of CoDistill-GRPO compared to GRPO. CoDistill-GRPO jointly trains a large and small model, reducing training cost by using rollouts generated by the small model. The large model provides initial hint tokens that the small model completes, and the resulting sequences are scored by a reward model. Because small-model rollouts are inexpensive, we generate many candidates ($M \gg G$) and apply downsampling. The large model’s log probabilities are used both to compute effective KD rewards for the small model and to score rollouts during downsampling. The small model is updated using  advantages $\widetilde{A}_i$ (see \Cref{eqn:small_advantage}), while the large model is updated using the original $A_i$.}
    \label{fig:hybrid_grpo}
\end{figure*}

\section{Background and Related Work}
\label{sec:background}

\subsection{Group Relative Policy Optimization}

Following the form in \Cref{eqn:rlvr}, the goal of GRPO is to find a policy that maximizes the following objective function:
\begin{align}
&\mathcal{L}_{\text{GRPO}}(\theta) = \mathbb{E}_{q \sim P(Q), \{o_i\}_{i=1}^G \sim \pi_{\theta_{\text{old}}}(\cdot | q)}\\
&\left[\frac{1}{G}\sum_{i=1}^G \frac{1}{|o_i|} \sum_{t=1}^{|o_i|}  \min \left\{ \frac{\pi_{\theta}(o_{i,t}|q, o_{i, <t})}{\pi_{\theta_{\text{old}}}(o_{i,t}|q, o_{i, <t})}  \widehat{A}_{i}, \text{clip}_{\epsilon}\left(\frac{\pi_{\theta}(o_{i,t}|q, o_{i, <t})}{\pi_{\theta_{\text{old}}}(o_{i,t}|q, o_{i, <t})}\right) \widehat{A}_{i} \right\} - \beta \cdot  \mbb{D}_{\text{KL}}[\pi_{\theta}(\cdot | q) \, \| \, \pi_{\text{ref}}(\cdot | q)] \right],
\label{eqn:grpo}
\end{align}
where $\pi_{\theta}$ and $\pi_{{\theta}_{\text{old}}}$ are the current and old policy models, respectively; $q$ is a question sampled from the dataset, and each $o_i$ is an output (or rollout) generated from the old policy $\pi_{{\theta}_{\text{old}}}$. Similar to the PPO objective~\citep{ppo}, there is a clipping operator introduced for the importance sampling ratio, where $\mathrm{clip}_{\epsilon}(x)$ outputs $x$ if $1-\epsilon \leq x \leq 1+\epsilon$, $1-\epsilon$ if $x < 1-\epsilon$, and $1+\epsilon$ if $x > 1+\epsilon$ for some $\epsilon \geq 0$. However, different from PPO, the advantages $\widehat{A}_{i}$ are computed by taking the normalized rewards across all generated $G$ outputs, as opposed to using generalized advantage estimation (GAE). Specifically, the advantages are computed as follows:
\begin{align} \label{eqn:advantage}
     \widehat{A}_i = \frac{r_i - \mathrm{mean}(\mbf{r})}{\mathrm{std}(\mbf{r})},
\end{align}
where $r_i$ is the reward for the output $o_i$ and $\mbf{r} = \{r_1, \ldots, r_G\}$ are the aggregated rewards across all $G$ rollouts. The KL divergence term is typically estimated using the following unbiased estimator:
\begin{align}
\label{eqn:kl_div}
   \widehat{\mbb{D}}_{\text{KL}}[\pi_{\theta}(o_i | q) \, \| \, \pi_{\text{ref}}(o_i | q)] = \frac{\pi_{\mathrm{ref}}(o_{i}|q)}{\pi_{\theta}(o_i | q)} - \log \left(\frac{\pi_{\mathrm{ref}}(o_{i}|q)}{\pi_{\theta}(o_{i} | q)}\right) -1.
\end{align}

\subsection{Related Work}

\paragraph{Improving GRPO.} There is a growing body of work aimed at improving GRPO~\citep{bnpo, liu2025understanding, gspo, dapo, shrivastava2025samplethinklessgroup, xu2025rolloutsusefuldownsamplingrollouts, zhang2025gvpo, cispo, zhao2026geometricmean}. CoDistill-GRPO can be combined with these methods, as our contribution primarily provides an algorithmic framework for co-distillation rather than a direct modification of GRPO. On the other hand, there are two more closely related works on improving GRPO for small models:
BREAD~\citep{zhang2025bread} and KDRL~\citep{kdrl}. As previously stated, both algorithms start from the assumption that a larger, more powerful language model is available to assist a smaller model.
In contrast to BREAD, CoDistill-GRPO starts with the large model generating a fixed number of initial tokens as hints for each rollout. This procedure is applied only during the early iterations to help align the two model distributions, enabling rollouts from the small model to be used to update the large model. 
KDRL shares the closest structural similarity to our algorithm and serves as our main comparison. It uses a KD-based method to improve small model performance while also maximizing the GRPO objective. However, \cite{kdrl} keeps the large model fixed: we show that jointly training models via CoDistill-GRPO yields better performance than merely applying GRPO with a KD-based reward from a static teacher model.


\paragraph{Knowledge Distillation.} KD is an effective method for reducing the computational overhead of LLMs by distilling knowledge from large models into smaller ones~\citep{hinton2015distillingknowledgeneuralnetwork, kim-rush-2016-sequence, agarwal2024onpolicy, gu2024minillm, xu2024surveyknowledgedistillationlarge, pmlr-v97-phuong19a, pmlr-v139-menon21a, kaplun2022knowledge, nagarajan2023on, rawat2024littlehelpgoeslong}. This is typically achieved by minimizing the KL divergence between the probability distributions of the student and teacher models. Our work is inspired by the on-policy distillation framework for LLMs proposed by \cite{agarwal2024onpolicy}, and similarly by \cite{lu2025onpolicydistillation}, who introduce an optimization objective that maximizes a reward augmented with an on-policy KL term for training a smaller student model. We remark that there also exists a line of work on ``self-distillation''~\citep{snell2022learningdistillingcontext, yuan2025selfrewardinglanguagemodels, allen-zhu2023towards}, but our goal is to directly use two different models to improve the performance of the smaller model.

\section{CoDistill-GRPO: An Efficient Co-Distillation Recipe for GRPO}
\label{sec:method}

In this section, we detail each component of CoDistill-GRPO, as illustrated by the flowchart in \Cref{fig:hybrid_grpo}. Throughout, we use $\theta$ and $\phi$ to denote the parameters of the large and small models, respectively, with $\pi_\theta$ and $\pi_\phi$ representing their corresponding policies.

Before diving into the specifics, we provide a high-level walkthrough of the algorithm. First, CoDistill-GRPO samples a batch of prompts $\mc{D}_B$. Each prompt $q \in \mc{D}_B$ is fed into the large model to generate $M$ partial outputs, each containing a maximum of $T$ tokens.
These serve as ``hints'' that are fed into the small model, which completes each sequence to a total length of $N \gg T$ tokens. Next, these rollouts are assigned two rewards: a standard reward $r$ for the large model, and an effective reward $\widetilde{r}$ for the small model, which incorporates a KD-based reward (see \Cref{eqn:small_reward}). Using $\widetilde{r}$ as a scoring metric, we downsample the rollouts from $M$ to $G$ to choose the most suitable outputs for updating both $\theta$ and $\phi$. Finally, after computing the advantages, the large model is updated using the standard rewards $r$ via the loss $\mc{L}_{\texttt{Large}}(\theta)$, while the small model is updated using the effective rewards $\widetilde{r}$ via $\mc{L}_{\texttt{Small}}(\phi)$. This entire process is summarized in \Cref{alg:algorithm}.

\Cref{sec:algorithm_breakdown} details each of these components, including the explicit formulations of $\mc{L}_{\texttt{Large}}(\theta)$ and $\mc{L}_{\texttt{Small}}(\phi)$, as well as the downsampling mechanism. Finally, \Cref{sec:theory} provides theoretical guarantees regarding the small model loss and its usage of the on-policy KD reward.

\begin{algorithm}[t!]
\caption{CoDistill-GRPO: Jointly Training Models via Co-Distillation}
\begin{algorithmic}[1]
\REQUIRE Initial policies $\pi_\phi$ and $\pi_\theta$; Prompts $\mathcal{D}$; Batch size $B$;  KD hyperparameter $\alpha$; Train steps $K$; Rollout sizes $M$ (before downsampling), $G$ (after downsampling); Hint size $T$; Hint steps $H$ 
\FOR{$k = 1, \ldots, K$}
    \STATE Sample a batch $\mathcal{D}_B$ from $\mathcal{D}$
\IF{$k \in [H]$}
    \STATE Sample hinted rollouts from large policy: $\{o_{i, 1:T}\}_{i=1}^M \sim \pi_\theta(\cdot|q)$ for each prompt $q \in \mathcal{D}_B$
    \STATE Complete rollouts with small policy: $\{o_{i, T+1:N}\}_{i=1}^M \sim \pi_\phi(\cdot | q, o_{i, 1:T})$ for each prompt $q \in \mathcal{D}_B$
\ELSE
    \STATE Sample rollouts with small policy: $\{o_{i, 1:N}\}_{i=1}^M \sim \pi_\phi(\cdot | q)$ for each prompt $q \in \mathcal{D}_B$
\ENDIF
    \STATE Compute rewards $\{r(o_i, q), \widetilde{r}(o_i, q)\}_{i=1}^M$ for every sampled response $o_i$ and prompt $q \in \mc{D}_B$ using \Cref{eqn:small_reward}
    \STATE Downsample using $\widetilde{r}$ as a scoring metric to obtain $\{r(o_i, q), \widetilde{r}(o_i, q)\}_{i=1}^G$
    \STATE Compute joint loss $\mc{L}(\theta, \phi) \coloneqq \mc{L}_{\texttt{Large}}(\theta) + \mc{L}_{\texttt{Small}}(\phi)$
    \STATE Take gradient $\nabla_{\theta, \phi}\mc{L}(\theta, \phi) = \nabla_{\theta}\mc{L}_{\texttt{Large}}(\theta) + \nabla_{\phi}\mc{L}_{\texttt{Small}}(\phi)$ and update policies $\pi_\theta$ and $\pi_\phi$
\ENDFOR
\RETURN Large policy $\pi_\theta$ and small policy $\pi_\phi$
\end{algorithmic}
\label{alg:algorithm}
\end{algorithm}

\subsection{The CoDistill-GRPO Components: Losses, Downsampling, and Design Choices}
\label{sec:algorithm_breakdown}

\subsubsection{GRPO Objective for the Small Model}
\label{sec:small_grpo_obj}

We begin with the GRPO objective function for the small model, $\mc{L}_{\texttt{Small}}(\phi)$. 
For simplicity, we start by introducing the case where the hinting length is $T = 0$; that is, the small model generates the entire set of outputs from some model $\pi_{\phi_{\text{old}}}$, and no hints are provided. The objective $\mc{L}_{\texttt{Small}}(\phi)$ has a very similar structure to \Cref{eqn:grpo}, where the differences are highlighted in \textcolor{blue}{blue}:
\begin{align}
&\mathcal{L}_{\text{Small}}(\phi) = \mathbb{E}_{q \sim P(Q), \textcolor{black}{\{o_i\}_{i=1}^G \sim \pi_{\phi_{\text{old}}}(\cdot | q)}} \\
&\left[\frac{1}{G}\sum_{i=1}^G \frac{1}{N}\sum_{t=1}^{N}  \min \left\{ \frac{\pi_{\phi}(o_{i,t}|q, o_{i, <t})}{\pi_{\phi_{\text{old}}}(o_{i,t} | q, o_{i, <t})}  \textcolor{blue}{\widetilde{A}_{i,t}}, \text{clip}_{\epsilon}\left(\frac{\pi_{\phi}(o_{i,t}|q, o_{i, <t})}{\pi_{\phi_{\text{old}}}(o_{i,t} | q, o_{i, <t})}\right) \textcolor{blue}{\widetilde{A}_{i,t}} \right\}\right],
\label{eqn:small_obj}
\end{align}
\begin{align}
\label{eqn:small_reward}
     \widetilde{r}_i &\coloneqq \widetilde{r}(q, o_{i}) = \underbrace{r(q, o_{i})}_{\text{original reward}}  + \underbrace{\alpha \cdot \frac{1}{N}\sum_{t=1}^{N}\log\left(\frac{\pi_{\theta}(o_{i,t}|q, o_{i, <t})}{\pi_{\phi}(o_{i,t} | q, o_{i, <t})}\right)}_{\text{on-policy KD reward}}, \\
     \widetilde{A}_{i, t} &\coloneqq \widetilde{A}_i = \widetilde{r}_i - \mathrm{mean}\left(\widetilde{\mbf{r}} \right), \label{eqn:small_advantage}
\end{align}
$\alpha > 0$ is a hyperparameter and $N > 0$ denotes the number of maximum completion tokens. 
The only difference lies in the rewards (and hence advantages): in addition to the original reward (e.g., accuracy for mathematical problems), the small model considers an additional reward inspired by on-policy KD \citep{agarwal2024onpolicy}. This additional reward incentivizes the small model to generate outputs that the large model would also generate, while penalizing those that the large model would not. This reward drives the probability distribution of the small model closer to that of the large model, achieving two goals: (i) improving the performance of the small model and (ii) promoting the use of the small model's rollouts to update the large model. 

The derivation of the effective reward $\widetilde{r}_i$ is also intuitive: if we consider the GRPO objective in \Cref{eqn:grpo} as a ``reward maximization'' problem with a KL regularization term between the small and large model, then we have the following:
\begin{align}
    \underset{\phi}{\mathrm{max}} \,\, &\mbb{E}_{q \sim P(Q)}  \left[ \mbb{E}_{o_i \sim \pi_{\phi}} \left[r(q, o_{i})\right] - \alpha \cdot  \mbb{D}_{\text{KL}}[\pi_{\phi}(\cdot | q) \| \pi_{\theta}(\cdot | q) ]  \right]\\
    = \underset{\phi}{\mathrm{max}} \,\,  &\mbb{E}_{q \sim P(Q), o_i \sim \pi_{\phi}} \biggl[r(q, o_{i}) -\alpha \cdot \log\left(\frac{\pi_{\phi}(o_{i}|q)}{\pi_{\theta}(o_{i} | q)}\right) \biggr].
\end{align}
Then, by averaging over all $N$ tokens, we obtain the effective reward in \Cref{eqn:small_reward} for each $o_i$. Following \cite{liu2025understanding}, we remove the standard deviation when computing advantages and normalize with $N$ instead of $|o_i|$, which has been shown to mitigate length bias and improve final performance.

\paragraph{Discussion on Design Choices.} There are a few design choices here that are worthy of discussion. First, notice that we use a different unbiased estimate of the KL divergence than in \Cref{eqn:kl_div} for the reward computation. We find that using the estimator in \Cref{eqn:kl_div} makes training unstable. Second, we incorporated the on-policy KD term into the reward instead of directly adding it to the objective function as done in \Cref{eqn:grpo} since we find that it generally yields better performance for the same values of $\alpha$. Finally, we do not include a regularizer from a reference policy, i.e., we set $\beta = 0$ in \Cref{eqn:grpo}. Prior work has shown that omitting this term often leads to better performance~\citep{liu2025understanding, hu2025openreasonerzeroopensourceapproach}, and we also aim for the small model to deviate from its initialization so that it can achieve sufficient improvement through GRPO.


\paragraph{Importance Reweighting with Hinting Tokens.} When $T>0$, the large model generates an initial hint with $T$ tokens. We have two options for importance sampling: (i) perform importance sampling on all tokens respectively as such:
\begin{align}
&\underbrace{
\sum_{t=1}^{T} 
\min \biggl(
  \frac{\pi_{\phi}(o_{i,t}\mid q, o_{i,<t})}{\pi_{\theta_{\text{old}}}(o_{i,t}\mid q, o_{i,<t})}
  \,\widetilde{A}_{i,t},
  \text{clip}_{\epsilon}\biggl(
    \frac{\pi_{\phi}(o_{i,t}\mid q, o_{i,<t})}{\pi_{\theta_{\text{old}}}(o_{i,t}\mid q, o_{i,<t})}
  \biggr)\widetilde{A}_{i,t}
\biggr)
}_{\text{importance sampling on initial tokens from large model}}
\\
&\quad\quad\quad\quad\quad\quad+
\underbrace{
\sum_{t=T+1}^{|o_i|}
\min \biggl(
  \frac{\pi_{\phi}(o_{i,t}\mid q, o_{i,<t})}{\pi_{\phi_{\text{old}}}(o_{i,t}\mid q, o_{i,<t})}
  \,\widetilde{A}_{i,t},
  \text{clip}_{\epsilon}\biggl(
    \frac{\pi_{\phi}(o_{i,t}\mid q, o_{i,<t})}{\pi_{\phi_{\text{old}}}(o_{i,t}\mid q, o_{i,<t})}
  \biggr)\widetilde{A}_{i,t}
\biggr)
}_{\text{importance sampling on completed tokens from small model}},
\label{eqn:small_sampling}
\end{align}
or (ii) ignore importance sampling on the hint tokens and perform importance sampling solely on those produced by the small model, which would amount to just using \Cref{eqn:small_sampling}. The second option is what is similarly used in the objective function of BREAD~\citep{zhang2025bread}.
However, we find that the second option generally makes training unstable, while the first option does not perform as well as treating these initial tokens as if they were generated by the small model. Hence, even for $T>0$, we use the objective in \Cref{eqn:small_obj}.

Furthermore, we also observe that providing these initial tokens as hints for \emph{all} iterations leads to instability. We hypothesize that this occurs because the small model often fails to complete rollouts correctly, which produces large gradients and results in unstable training. For this reason, we provide hints only during the early iterations as a warm start for the small model. We generally find that providing hints during the first epoch is sufficient to enhance rollout quality while maintaining overall algorithm stability.

\subsubsection{GRPO Objective for the Large Model}
\label{sec:large_model_obj}

The GRPO objective for the large model $\mc{L}_{\texttt{Large}}(\theta)$ is straightforward: it must account for the fact that rollouts are generated by the small model $\pi_{\phi_{\text{old}}}$. We handle this by treating the problem as an off-policy RL problem and applying importance sampling with respect to the small model's token-probability distribution, which leads to the following objective:
\begin{align}
&\mathcal{L}_{\text{Large}}(\theta) = \mathbb{E}_{q \sim P(Q), \textcolor{blue}{\{o_i\}_{i=1}^G \sim \pi_{\phi_{\text{old}}}(\cdot | q)}} \\
&\left[\frac{1}{G}\sum_{i=1}^G \frac{1}{N}\sum_{t=1}^{N}  \min \left\{ \frac{\pi_{\theta}(o_{i,t}|q, o_{i, <t})}{\textcolor{blue}{\pi_{\phi_{\text{old}}}(o_{i,t} | q, o_{i, <t})}} \hat{A}_{i,t}, \text{clip}_{\epsilon}\left(\frac{\pi_{\theta}(o_{i,t}|q, o_{i, <t})}{\textcolor{blue}{\pi_{\phi_{\text{old}}}(o_{i,t} | q, o_{i, <t})}}\right) \hat{A}_{i,t} \right\} \right].
\label{eqn:large_obj}
\end{align}
\Cref{eqn:large_obj} has the same structure as the GRPO objective in \Cref{eqn:grpo}, but now accounts for the rollouts generated via the small model. Similar to \Cref{sec:small_grpo_obj}, we omit the regularizer with the reference model and normalize with $N$ instead of $|o_i|$.

\subsubsection{Downsampling Rollouts}
\label{sec:rejection}

Lastly, we discuss the downsampling block in \Cref{fig:hybrid_grpo}. The main idea is as follows: since all rollouts are generated using the more inexpensive small model, we can produce $M > G$ rollouts and then select a subset of $G$ rollouts that will likely lead to the most informative  updates. 
\citet{xu2025rolloutsusefuldownsamplingrollouts} show that choosing the rollouts with the top and bottom $G/2$ rewards yields better performance than selecting the top-$G$ rollouts with the highest rewards. Intuitively, this maximizes reward variance and exposes the model to both strong and weak trajectories. Similarly, CoDistill-GRPO selects the top and bottom $G/2$ rollouts based on the effective reward in \Cref{eqn:small_reward}, which incorporates the KD score, as these selected rollouts are also used to update the large model.

\subsection{Theoretical Results}
\label{sec:theory}

Here, we present a couple of theoretical results regarding the small model loss $\mc{L}_{\texttt{Small}}(\phi)$. First, although the small model uses an effective reward that incorporates the on-policy KD term, we show that it still enjoys desirable statistical properties such as an unbiased gradient estimate.
\begin{lemma}
\label{lem:unbiased}
    Let $\eta > 0$ and $G>1$ denote the learning rate and group size, respectively. 
    Suppose that for each prompt $q \sim P(Q)$, the outputs are sampled on-policy $\{o_i\}_{i=1}^G \sim \pi_{\phi}(\cdot \, | \, q)$  
    and $\pi_\phi$ is continuously differentiable.
    If the policy $\pi_\phi$ is updated using the effective learning rate $\eta_{\text{eff}} \coloneqq \left(\frac{G}{G-1}\right) \cdot \eta$ with advantage estimates in \Cref{eqn:small_reward}, then the gradient of \Cref{eqn:small_obj} is an  unbiased estimate of the true policy gradient:
    \begin{align}
        \eta_{\text{eff}} \cdot \mbb{E}\left[\nabla_{\phi} \mc{L}_{\texttt{Small}}(\phi)\right] = \eta \cdot \nabla_{\phi} J(\phi),
    \end{align}
where $J(\phi) \coloneqq \mbb{E}_{q\sim P(Q), o_i\sim \pi_{\phi}(\cdot | q)} \left[ \widetilde{r}_{\phi}(q, o_i)\right]$ is the true reward objective.
\end{lemma}

Note that \Cref{lem:unbiased} makes two assumptions: the outputs are sampled on-policy, meaning that importance reweighting is not performed; and only $G$ outputs are sampled, meaning that the downsampling procedure is skipped. These assumptions simplify the analysis, allowing us to demonstrate that the algorithm's properties hold based solely on its core components. 
Then, using \Cref{lem:unbiased}, we prove that the gradient of $\mc{L}_{\texttt{Small}}(\phi)$ can be decomposed into two terms: a direction towards the original reward and a direction towards the large model. This shows that when $\alpha > 0$, the effective reward drives the small model towards the large model, explaining why the effective reward improves small model performance beyond standard GRPO.

\begin{theorem}[Gradient Decomposition]
\label{thm:gradient}
    Let $\alpha >0$ and $N > 0$ denote the KD regularization parameter and number of maximum completion tokens, respectively. Under the same conditions as \Cref{lem:unbiased}, for each prompt $q \sim P(Q)$, the expected gradient of $\mc{L}_{\texttt{Small}}(\phi)$ admits the following decomposition:
    \begin{align*}
        \eta_{\text{eff}} \cdot \mbb{E}\left[\nabla_{\phi} \mc{L}_{\texttt{Small}}(\phi) \, | \, q \right] = \eta\cdot \underbrace{\nabla_{\phi} \mbb{E}_{o_i \sim \pi_{\phi}(\cdot | q)}\left[r(q, o_i)\right]}_{\text{direction towards original reward}} - \eta \cdot \frac{\alpha}{N} \underbrace{\nabla_{\phi} \mbb{D}_{\texttt{KL}}\left( \pi_{\phi}(\cdot | q) \| \pi_{\theta}(\cdot | q)\right)}_{\text{direction towards large model}}.
    \end{align*}
\end{theorem}
When $\alpha > 0$, the on-policy KD reward introduces an additional direction towards the large model, whereas with $\alpha = 0$ we would only have the direction towards maximizing the original reward. This is similar in spirit to \cite{kdrl}, but adapted for the effective reward setting and rigorously derived. All proofs are deferred to \Cref{sec:deferred_proofs}.

\section{Experimental Results}
\label{sec:experiments}

\subsection{Experimental Setup}

We evaluate the effectiveness of CoDistill-GRPO on both Qwen and Llama model families across four mathematical benchmarks: Minerva, MATH500, AMC2024, and OlympiadBench. For the Qwen models, we use Qwen2.5-Math-1.5B as the small model and Qwen2.5-Math-7B as the large model. These models were supervised fine-tuned on a mathematical dataset, making them well-suited for mathematical reasoning tasks and consistent with prior work. For the Llama family, we use Llama3.2-1B-Instruct and Llama3.1-8B-Instruct. All models are trained on the MATH dataset~\citep{hendrycks2021measuring}, which consists of 12,500 samples, for a total of 8 epochs for Qwen models and 2 epochs for Llama models.
We consider two verifiable reward functions—accuracy and format rewards—as done by \cite{deepseek-math}. A detailed list of hyperparameters and training configurations is deferred to \Cref{sec:training_dets}.

\begin{table}[t!]
\centering
\resizebox{\textwidth}{!}{
\begin{tabular}{lrrrr|r}
\toprule
Algorithm & Minerva & MATH500 & AMC & OlympiadBench & Average\\
\midrule
Qwen2.5-Math-1.5B & $20.37 \pm 0.02$ & $54.92 \pm 0.01$ & $32.00 \pm 0.05$ & $23.14 \pm 0.01$ & $32.61$ \\
\midrule
+GRPO & $25.88 \pm 0.01$ & $72.12 \pm 0.01$ & $55.50 \pm 0.07$ & $35.05 \pm 0.00$ & $47.14$ (\textcolor{blue}{$+14.53$}) \\
+KDRL w/ 7B & $27.13 \pm 0.02$ & $74.24 \pm 0.01$ & $54.00 \pm 0.05$ & $\mathbf{37.78} \pm 0.02$ & $48.29$ (\textcolor{blue}{$+15.68$}) \\
+CoDistill-GRPO ($\alpha = 1$, $T=0$) & $\underline{29.12} \pm 0.01$ & $\mathbf{75.04} \pm 0.01$ & $54.50 \pm 0.04$ & $\underline{36.39} \pm 0.01$ & $48.76$ (\textcolor{blue}{$+ 16.15$})\\
+CoDistill-GRPO ($\alpha = 2$, $T=0$) & $\mathbf{31.99} \pm 0.02$ & $74.28 \pm 0.01$& $\underline{55.50} \pm 0.06$ & $35.14 \pm 0.01$ &  $49.23$ (\textcolor{blue}{$+ 16.62$})\\
+CoDistill-GRPO  ($\alpha = 1, T=50$) & $28.30\pm 0.01$ & $\underline{74.52}\pm 0.01$& $\mathbf{60.50} \pm 0.03$ & $36.12 \pm 0.02$ &  $49.86$ (\textcolor{blue}{$+ 17.25$})\\

\bottomrule
\end{tabular}
} 
\caption{Results for the small model Qwen2.5-Math-1.5B using CoDistill-GRPO compared to GRPO and KDRL. While all methods improve upon the base model, CoDistill-GRPO yields the largest overall improvement when averaged across datasets. Best results are shown in \textbf{bold}, and second-best results are \underline{underlined}.
}
\label{tab:qwen_small_results}
\end{table}

\begin{figure}[t!]
    \centering
    \includegraphics[width=0.95\linewidth]{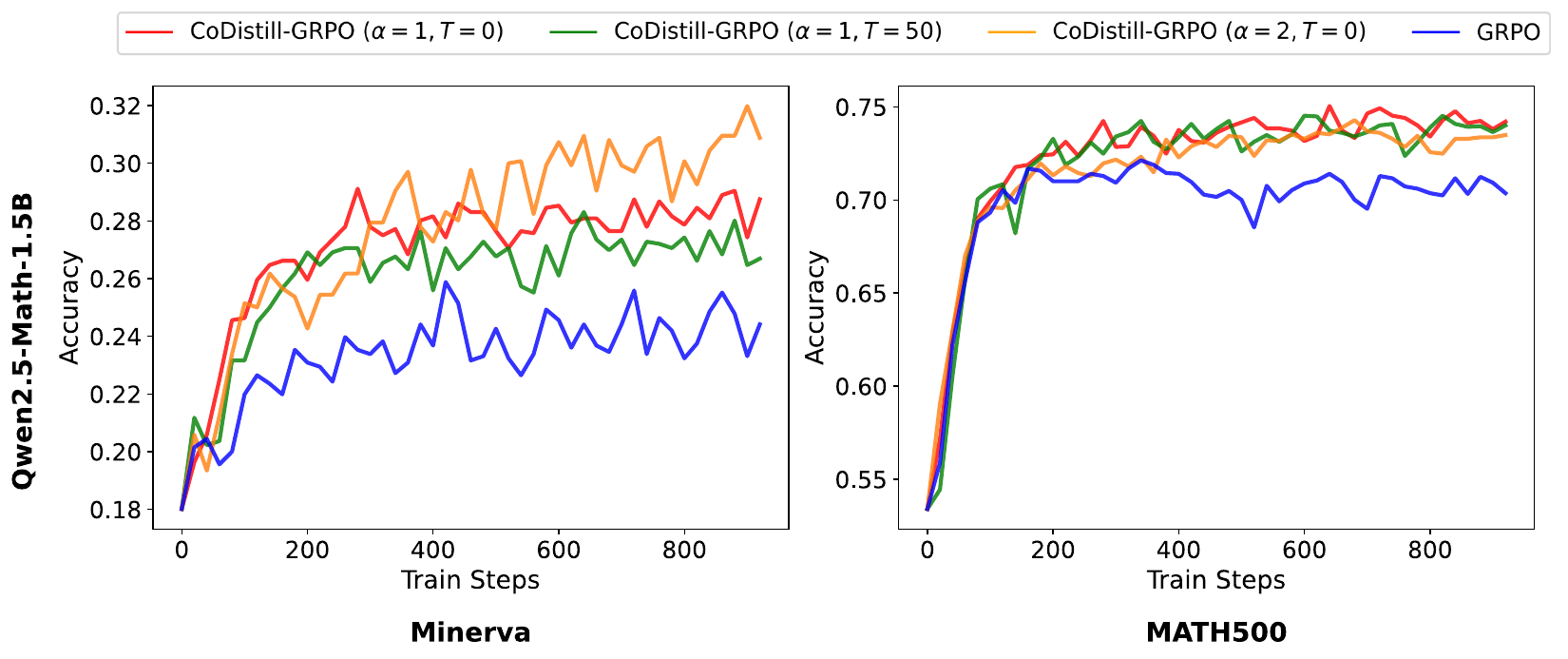}
    \caption{Test accuracy of the Qwen2.5-Math-1.5B model across training iterations on the Minerva and MATH500 datasets. Compared to GRPO, CoDistill-GRPO learns at a faster rate and can achieve a higher final test accuracy.}
    \label{fig:small_minerva_math}
\end{figure}

We compare CoDistill-GRPO, using different values of $\alpha$ and initial hint lengths $T$, against two baselines: GRPO and KDRL~\citep{kdrl}. For KDRL, we first run GRPO on the large model, select several well-performing checkpoints (including the base model), and then train the small model using the effective on-policy KD reward. We report the results with the best performing checkpoint. Note that although KDRL freezes the large model during training, it still requires running GRPO beforehand, making it the most time-consuming method. We also fix $\alpha = 1$ for KDRL, as we empirically found that using a smaller fixed value performed better than using an annealed schedule.
For CoDistill-GRPO, we provide hints, i.e., $T>0$, only during the first epoch and set $T=0$ for all subsequent iterations. For rollout generation, we use downsampling by generating $M = 14$ rollouts and updating with a subset of $G = 8$ rollouts. 
This was also done for KDRL to provide a fair comparison. For GRPO, we generate and update using only $G = 8$ rollouts.

\begin{figure}[h!]
    \centering
    \includegraphics[width=0.425\linewidth]{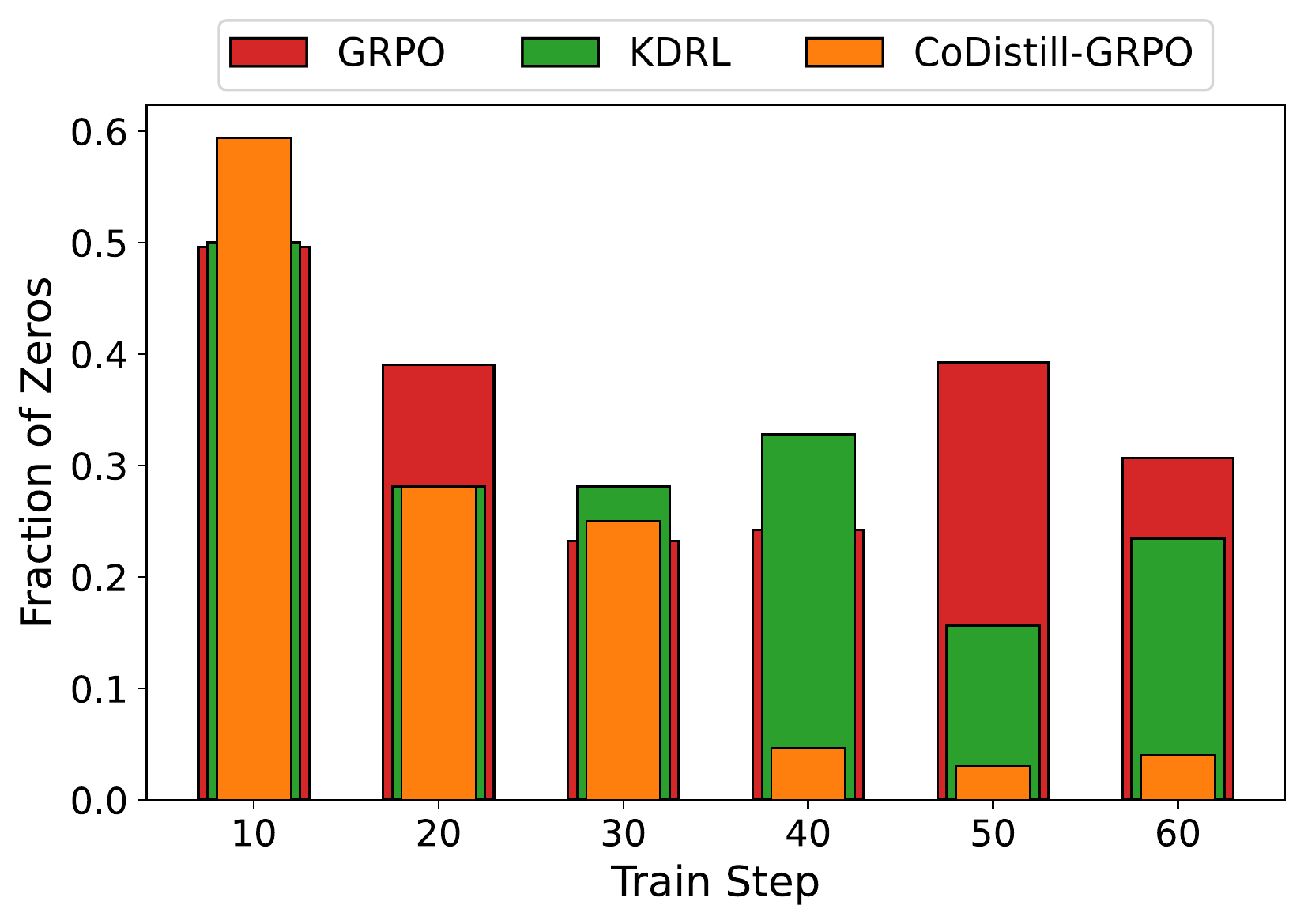}
    \caption{Fraction of zero accuracy rewards on the training dataset across iterations on the MATH dataset~\citep{hendrycks2021measuring} for the Qwen2.5-Math-1.5B. Compared to GRPO and KDRL, CoDistill-GRPO has a lesser fraction of zero accuracy rewards in the earlier iterations.}
    \label{fig:frac_zeros_distill}
\end{figure}

\subsection{Results on Qwen Models}

In \Cref{tab:qwen_small_results}, we present quantitative results on mathematical benchmarks for the Qwen2.5-Math-1.5B model. We report mean accuracy and standard deviation across five random seeds. While all methods improve upon the base model, CoDistill-GRPO achieves a notable improvement over GRPO with the small model alone, with an increase of approximately $2.7$ percentage points on average accuracy. 
The small model trained with CoDistill-GRPO also reaches a remarkable $32\%$ accuracy on the Minerva dataset as opposed to $26\%$ with GRPO.
On the other hand, while KDRL improves upon standard GRPO, its gains are marginal, yielding an average performance increase of $1.15$ percentage point. We hypothesize that CoDistill-GRPO outperforms KDRL because jointly updating both models mitigates the previously discussed discrepancy between the large teacher and small student models. We also remark that while CoDistill-GRPO with large model hints (i.e., $T>0$) achieves the best average performance increase, there are datasets where it does not improve upon the $T=0$ baseline. This implies that hinting is not universally beneficial, as the small model is often unable to complete the initial rollouts generated by the large model correctly. We leave the investigation of optimal hinting strategies to enhance small model rollouts for future work.

In \Cref{fig:small_minerva_math}, we show the evolution of test accuracy on Minerva and MATH500 throughout training, and defer the AMC and OlympiadBench plots to \Cref{sec:add_figures}.
Referring back to \Cref{fig:frac_zeros}, we show in \Cref{fig:frac_zeros_distill} that, compared to GRPO and KDRL, CoDistill-GRPO yields a smaller fraction of zero-accuracy rewards. We attribute this reduction to the  gains observed in the quantitative results.  Finally, in \Cref{sec:ablation}, we provide additional ablation studies using Qwen models. These ablation studies include several settings in which CoDistill-GRPO was unstable (e.g., providing more hints, or using the KL divergence in \Cref{eqn:kl_div}).

\subsection{Results Beyond Qwen Models}
\label{sec:llama_results}

We also showcase CoDistill-GRPO on Llama models. In \Cref{tab:llama_results}, we show that CoDistill-GRPO provides improvements over GRPO. Although both methods outperform the base model, the gains are considerably smaller than those observed with Qwen. This observation is consistent with findings in the literature (see results on Llama models reported by \citet{liu2025understanding}). We also find that training Llama with GRPO is typically unstable, making benchmarking more challenging than with Qwen. In contrast, CoDistill-GRPO appeared substantially more stable and achieved consistent improvements over GRPO. We leave a more thorough investigation to future work.

\begin{table}[ht]
\centering
\resizebox{\textwidth}{!}{
\begin{tabular}{lrrrr|r}
\toprule
Algorithm & Minerva & MATH500 & AMC & OlympiadBench & Average\\
\midrule
Llama3.2-1B-Instruct      & $6.465 \pm 0.55$   & $20.56 \pm 2.20$  & $5.500 \pm 2.45$  & $4.328 \pm 0.22$         & $9.213$          \\
\midrule
+GRPO                & $5.664 \pm 0.95$    & $22.36 \pm 2.47$   & $6.500 \pm 4.06$  & $5.038 \pm 0.57$         & $9.891$ (\textcolor{blue}{$+ 0.678$})          \\
+CoDistill-GRPO ($\alpha = 1$, $T=0$)         &  $6.396 \pm 1.50$   &  $22.30 \pm 1.26$  & $10.50 \pm 1.87$ & $5.086\pm 0.45$     &   $11.07$ (\textcolor{blue}{$+ 1.857$})      \\  
\bottomrule
\end{tabular}
}
\caption{Results for the small model Llama3.2-1B-Instruct using CoDistill-GRPO compared to baselines on Minerva, MATH500, AMC2024, and OlympiadBench. CoDistill-GRPO improves upon GRPO when averaged across all benchmarking datasets.}
\label{tab:llama_results}
\end{table}

\begin{table}[t!]
\centering
\resizebox{\textwidth}{!}{
\begin{tabular}{lrrrr|r}
\toprule
Algorithm & Minerva & MATH500 & AMC & OlympiadBench & Average\\
\midrule
Qwen2.5-Math-7B & $24.04 \pm 0.01$ & $66.12 \pm 0.01$ & $47.00 \pm 0.04$ & $34.19 \pm 0.01$ & $42.84$ \\
\midrule
+GRPO & $\mathbf{38.75} \pm 0.01$ & $80.12 \pm 0.01$ & $\mathbf{70.00} \pm 0.01$ & $\underline{42.73} \pm 0.01$ & $57.90$ (\textcolor{blue}{$+ 15.06$}) \\
+CoDistill-GRPO ($\alpha=1$, $T=0$) & $34.56 \pm 0.01$ & $77.44 \pm 0.01$ & $64.00 \pm 0.07$ & $38.90 \pm 0.01$ & $53.73$ (\textcolor{blue}{$+10.89$}) \\
+CoDistill-GRPO ($\alpha=1$, $T=0$) w/ CT & $\underline{38.16} \pm 0.01$ & $79.04 \pm 0.00$ & $\underline{69.50} \pm 0.03$ & $42.40 \pm 0.01$ & $57.27$ (\textcolor{blue}{$+ 14.43$})\\
+CoDistill-GRPO ($\alpha=1$, $T=50$) & $35.51 \pm 0.01$ & $77.04 \pm 0.01$ & $62.00 \pm 0.02$ & $40.41 \pm 0.01$ & $53.74$ (\textcolor{blue}{$+10.90 $})\\
+CoDistill-GRPO ($\alpha=1$, $T=50$) w/ CT & $37.28 \pm 0.02$ & $\mathbf{81.12} \pm 0.01$ & $63.50 \pm 0.05$ & $\mathbf{42.96} \pm 0.01$ & $56.22$ (\textcolor{blue}{$+13.38 $}) \\
+CoDistill-GRPO ($\alpha=2$, $T=0$) & $35.15 \pm 0.01$ & $76.52 \pm 0.01$ & $62.00 \pm 0.07$ & $38.84 \pm 0.01$ & $53.13$ (\textcolor{blue}{$+10.29 $})\\
+CoDistill-GRPO ($\alpha=2$, $T=0$) w/ CT & $36.00 \pm 0.01$ & $\underline{80.20} \pm 0.00$ & $63.00 \pm 0.02$ & $41.01 \pm 0.01$ & $55.05$ (\textcolor{blue}{$+12.21 $}) \\
\bottomrule
\end{tabular}
}
\caption{Results for the large model Qwen2.5-Math-7B using CoDistill-GRPO compared to baselines on Minerva, MATH500, AMC2024, and OlympiadBench. While all methods improve upon the base model, GRPO yields the largest overall improvement. However, with CT, CoDistill-GRPO can closely match the performance of GRPO while still reducing total training time. Best results are shown in \textbf{bold}, and second-best results are \underline{underlined}.}
\label{tab:qwen_large_results}
\end{table}

\begin{figure}[t!]
    \centering
    \includegraphics[width=0.95\linewidth]{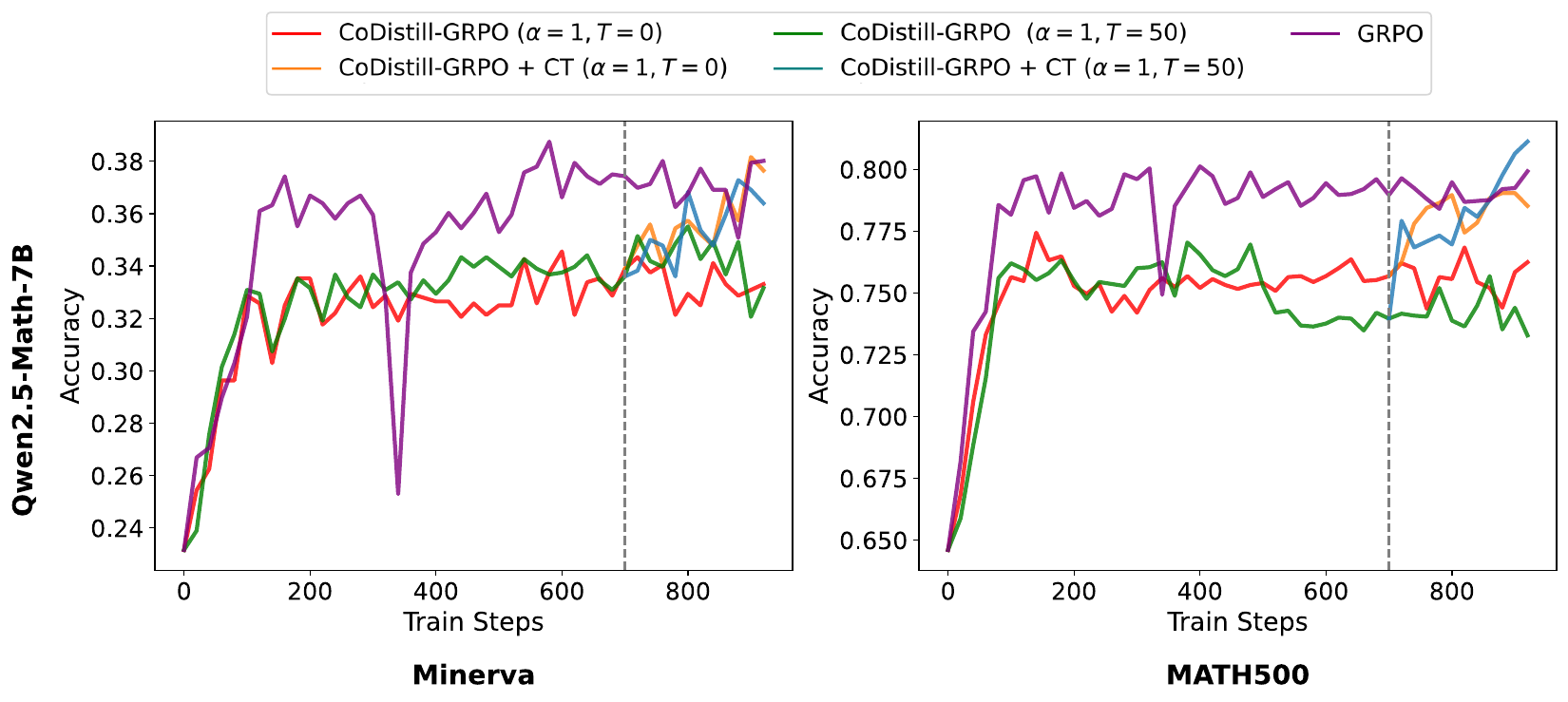}
    \caption{Test accuracy of the Qwen2.5-Math-7B model across training iterations on the Minerva and MATH500 datasets. While GRPO obtains the best performance averaged across all benchmark datasets, CoDistill-GRPO with continued training (CT) closely matches GRPO even while using rollouts from the smaller model.}
    \label{fig:large_minerva_math}
\end{figure}

\subsection{How Much Does CoDistill-GRPO Improve the Large Model?}

Since CoDistill-GRPO updates both models, an important question is how much improvement we can expect in the large model relative to GRPO. Because the large model is updated using rollouts generated by the small model, we can expect a reduction in training time.
We compare standard GRPO on the Qwen2.5-Math-7B large model against both CoDistill-GRPO and CoDistill-GRPO with continued training (CT). CT serves as a post-processing step: we take a large model checkpoint from CoDistill-GRPO and run a few additional iterations of standard GRPO. This allows the model to use its own rollouts to boost performance, rather than relying only on the small model's rollouts.

In \Cref{tab:qwen_large_results}, we present the quantitative results, where we observe that CoDistill-GRPO with CT outperforms standard GRPO on the MATH500 and OlympiadBench, while lacking behind on Minerva and AMC. Of course, CoDistill-GRPO without CT does not outperform standard GRPO, as small model rollouts alone are not sufficient for substantial improvements. In \Cref{fig:large_minerva_math}, we show the evolution of test accuracy during training, indicating where CT is introduced and demonstrating how much it improves the large model's performance.

Finally, in \Cref{fig:large_model_efficiency} (left), we illustrate the time savings of CoDistill-GRPO compared to standard GRPO. By measuring the execution time of a single training iteration, we demonstrate an 18\% reduction per step for CoDistill-GRPO when excluding CT and downsampling. 
In \Cref{fig:large_model_efficiency} (right), we present the total wall-clock time required strictly for rollout generation over a full training run. Here, standard GRPO generates $G=8$ rollouts, whereas CoDistill-GRPO generates $M=14$ rollouts prior to downsampling. We show that even when factoring in the overhead of CT, hinting (since $T=50$), and downsampling, CoDistill-GRPO provides overall training-time savings during rollout generation.


\begin{figure}[t!]
    \centering

    \begin{subfigure}{0.45\linewidth}
        \centering
        \includegraphics[width=\linewidth]{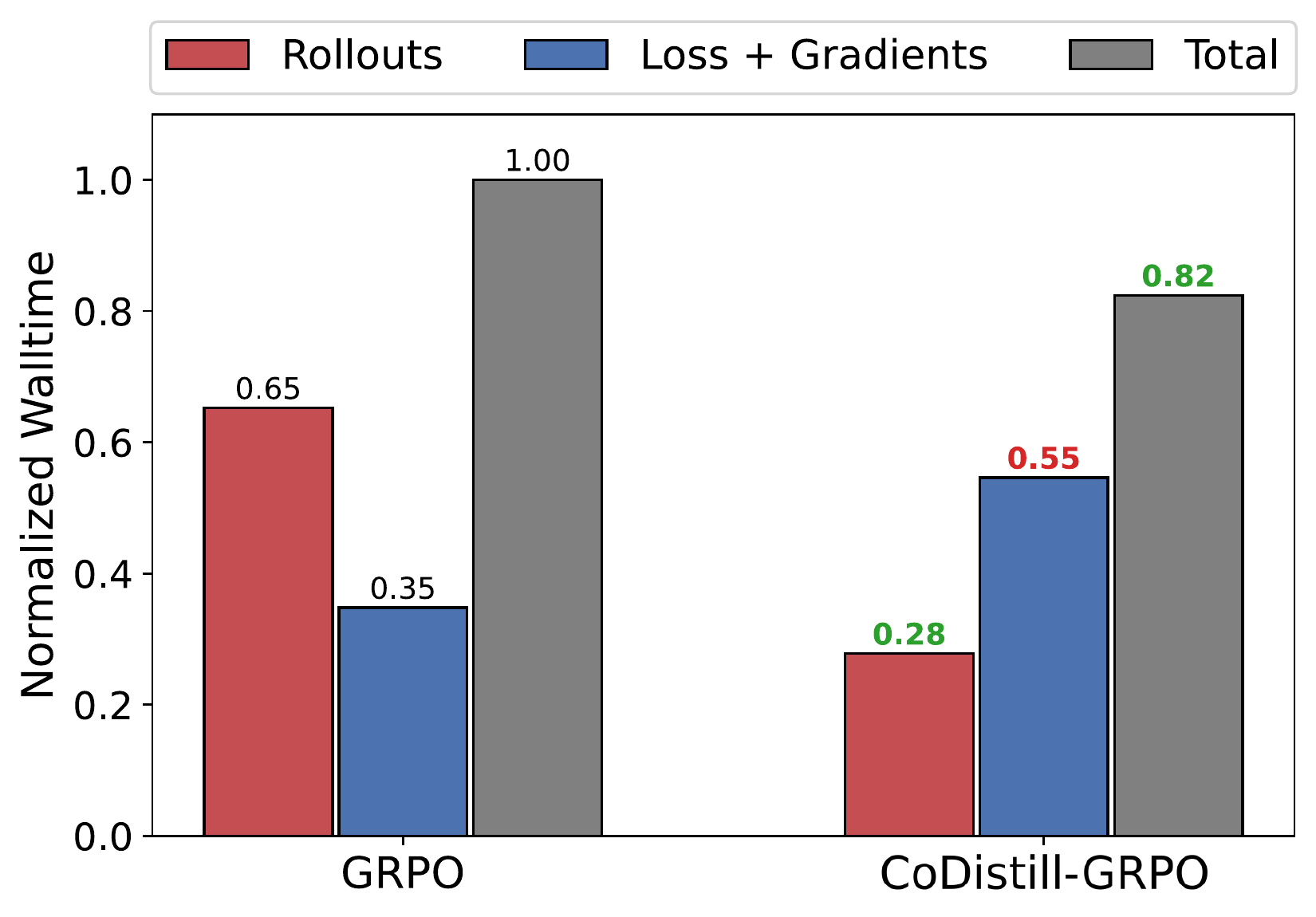}
        \label{fig:theoretical_savings}
    \end{subfigure}
    \hfill
\begin{subfigure}{0.45\linewidth}
    \centering
    \raisebox{0.1in}{
        \includegraphics[width=0.85\linewidth]{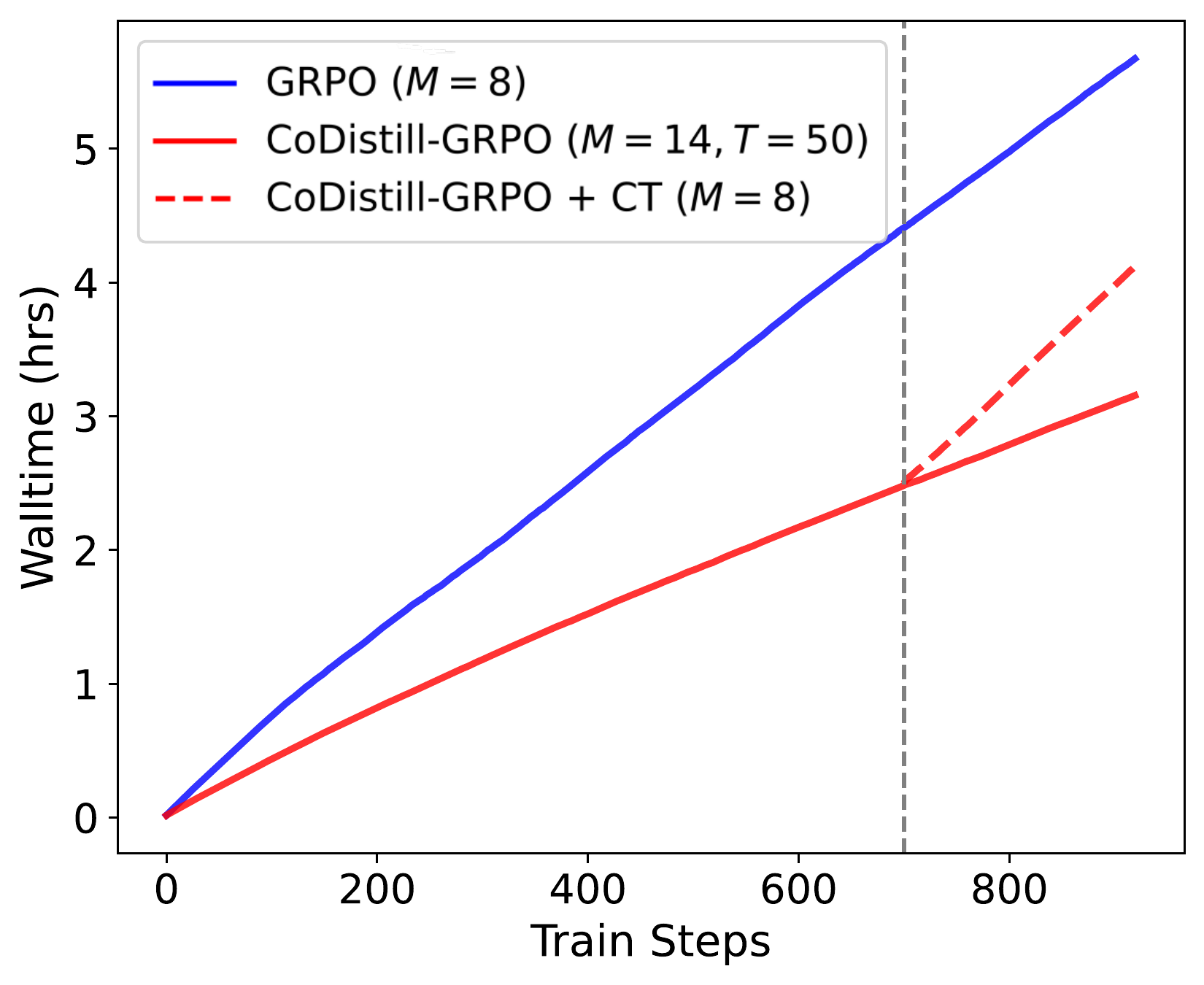}
    }
        \label{fig:rollout_times}
    \end{subfigure}

    \caption{Comparison of training efficiency between CoDistill-GRPO and GRPO. Left: Demonstration of the potential time savings using CoDistill-GRPO over GRPO to train a Qwen2.5-Math-7B model on a single 80GB H100 GPU without CT. Even though CoDistill-GRPO updates two models on every iteration (the small and the large), it achieves a reduction in training time per iteration due rollout generation. Right: Cumulative walltime in hours for rollout generation. Even with CT, CoDistill-GRPO has a reduction in training time on Qwen2.5-Math-7B.}
    \label{fig:large_model_efficiency}
\end{figure}

\section{Conclusion}
\label{sec:conclusion}
This work presented CoDistill-GRPO, a co-distillation algorithm designed to improve GRPO performance on small language models. CoDistill-GRPO alleviates the assumption that a powerful, pretrained large model must already exist as an oracle to provide hints to the smaller model by jointly training two models in tandem, while relying on rollouts generated by the small model. To achieve this, CoDistill-GRPO combines initial hint tokens with an on-policy KD reward and a downsampling mechanism. These components enable CoDistill-GRPO to substantially improve the performance of the small model relative to existing baselines, and even enhance the performance of the large model with continued training (CT), all while reducing overall training time.
There remain several promising directions for future work. One important avenue is to improve the performance of the large model so that it can outperform GRPO without CT, which may in turn further enhance the performance of the small model. Another avenue is to explore how speculative decoding could be integrated into the algorithm: if the two models are sufficiently aligned via co-distillation, this may enable faster inference while minimizing discrepancies between the large and small model completions.

\clearpage

\bibliography{main}

\clearpage

\onecolumn
\par\noindent\rule{\textwidth}{1pt}
\begin{center}
{\Large \bf Appendix}
\end{center}
\vspace{-0.1in}
\par\noindent\rule{\textwidth}{1pt}
\appendix

\tableofcontents

\clearpage

\section{Training Details}
\label{sec:training_dets}

\begin{table}[h!]
\centering
\begin{tabular}{l|c|c}
\toprule
\textbf{Hyperparameter} & Qwen Models & Llama Models \\
\midrule
 Learning Rate & $5\times 10^{-6}$ & $1\times 10^{-7}$ \\
\midrule
 Learning Rate Schedule & Cosine & Constant \\
\midrule
 Learning Rate Warmup Steps & $5\%$  & $0\%$  \\
\midrule
 Max Completion Length & $3072$ & $3072$ \\
 \midrule
 Max Prompt Length & $1024$ & $1024$ \\
 \midrule
 Batch Size Per GPU & $64$ & $64$ \\
 \midrule
 Number of Rollouts & $M=14, G=8$ & $M=12, G=8$ \\
  \midrule
 Total Epochs & $8$ & $2$ \\
\bottomrule
\end{tabular}
\caption{Hyperparameter configurations used to train both Qwen and Llama models. We use a constant learning rate with no warmup steps for Llama models, as this yielded the best performance. For CoDistill-GRPO, we generate an initial $M$ traces, for which $G$ are chosen using the downsampling method. For GRPO, we generate and update using $G$ rollouts.}
\label{tab:hyperparameters}
\end{table}

In this section, we present the omitted training details that support the experiments in \Cref{sec:experiments}. All experiments were conducted using 8 H100 GPUs unless otherwise specified. We report the hyperparameter settings in \Cref{tab:hyperparameters}. The same set of hyperparameters was used for both the small and large models, as well as for CoDistill-GRPO and GRPO, as this configuration yielded the best performance.
For the number of rollouts, $M$ denotes the number of rollouts initially generated for CoDistill-GRPO, while $G$ denotes the number of rollouts selected after downsampling. For GRPO, we directly generate $G$ rollouts and perform updates using these $G$ samples. In experiments with Llama models, we found that using a constant learning rate achieved the best performance on the test datasets. We also observed that increasing the number of epochs was detrimental, as longer training led to instability.

Following DeepSeek-R1~\citep{deepseek-math, deepseekai2025deepseekr1incentivizingreasoningcapability}, we add an additional format reward for Qwen models to encourage adherence to a specified output format. The format used for the Qwen models is shown in \Cref{tab:deepseek-template}. The model receives a reward of $0.25$ for each of the \texttt{</think>} and \texttt{</answer>} tokens, for a total possible format reward of $1.00$. For accuracy rewards, each correct answer is also assigned a reward of $1.00$.
For the Llama models, we use the chat template shown in \Cref{tab:llama-template}, as we observed that this template provides a substantial performance improvement compared to the template used for Qwen. However, we omit format rewards for Llama models and include only the accuracy reward to reflect this difference.

\begin{table}[t]
\centering
\renewcommand{\arraystretch}{1.2}
\begin{tabular}{p{0.9\linewidth}}
\hline
A conversation between User and Assistant. The user asks a question, and the Assistant solves it.
The assistant first thinks about the reasoning process in the mind and then provides the user
with the answer. The reasoning process and answer are enclosed within
\texttt{<think></think>} and \texttt{<answer></answer>} tags, respectively, i.e.,
\texttt{<think> reasoning process here </think>} and
\texttt{<answer> answer here </answer>}.
User: $\{$\textcolor{red}{prompt}$\}$. Assistant:
\\
\hline
\end{tabular}
\caption{Chat template for DeepSeek-R1-Zero~\citep{deepseekai2025deepseekr1incentivizingreasoningcapability} that is used to train the Qwen models for both GRPO and CoDistill-GRPO. The \textcolor{red}{prompt} is replaced with the specific reasoning question during training.}
\label{tab:deepseek-template}
\end{table}

\begin{table}[t]
\centering
\renewcommand{\arraystretch}{1.2}
\begin{tabular}{@{}p{0.9\linewidth}@{}}
\hline
Please reason step by step, and put your final answer within \texttt{\textbackslash boxed\{\}}. \\
Question: \{\textcolor{red}{prompt}\}. Answer: \\
\hline
\end{tabular}
\caption{Chat template used to train the Llama models for both GRPO and CoDistill-GRPO. The \textcolor{red}{prompt} is replaced with the specific reasoning question during training.}
\label{tab:llama-template}
\end{table}

\section{Ablation Studies}
\label{sec:ablation}

In this section, we present ablation studies investigating the algorithmic choices made in CoDistill-GRPO. All ablation studies were performed on the Qwen models. In \Cref{fig:vanilla_hybrid_fail}, we demonstrate how the KD reward and rejection-sampling components are crucial for simultaneously training both models. We observe that without these two components, the large model becomes unstable during training when evaluated on Minerva and MATH500. This shows that simply using small-model rollouts to update the large model is insufficient.

In \Cref{fig:initial_trace_ign}, we study the effect of using initial traces (or hints) from the large model throughout all training iterations, as opposed to only during the early iterations. On the left, we plot the training accuracy of the small model over time. When the training accuracy drops, the large model is updated using rollouts with low accuracy; consequently, the large model’s performance degrades significantly and fails to recover. In contrast, the setting with $T=0$ trains stably.
On the right, we show that ignoring the initial traces provided by the large model in the importance-sampling update for the small-model objective also leads to unstable training. Finally, in \Cref{fig:skl_failure}, we show that using the KL-divergence definition in \Cref{eqn:kl_div} results in training instability.

\begin{figure}[h!]
    \centering
    \includegraphics[width=0.9\linewidth]{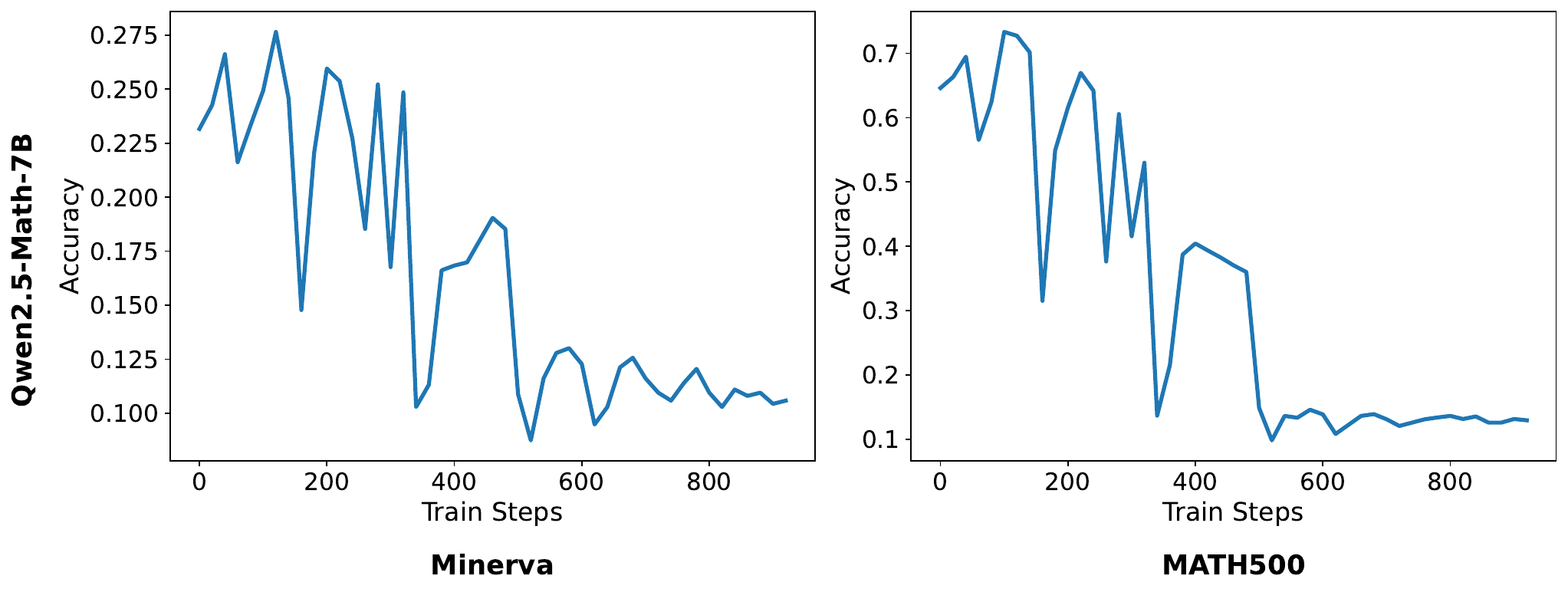}
    \caption{Using small model rollouts to update the policy model with no knowledge distillation or downsampling. These two components that make up CoDistill-GRPO are critical for making simultaneous training of models feasible.}
    \label{fig:vanilla_hybrid_fail}
\end{figure}

\begin{figure}[h!]
    \centering

    \begin{subfigure}{0.45\linewidth}
        \centering
        \includegraphics[width=\linewidth]{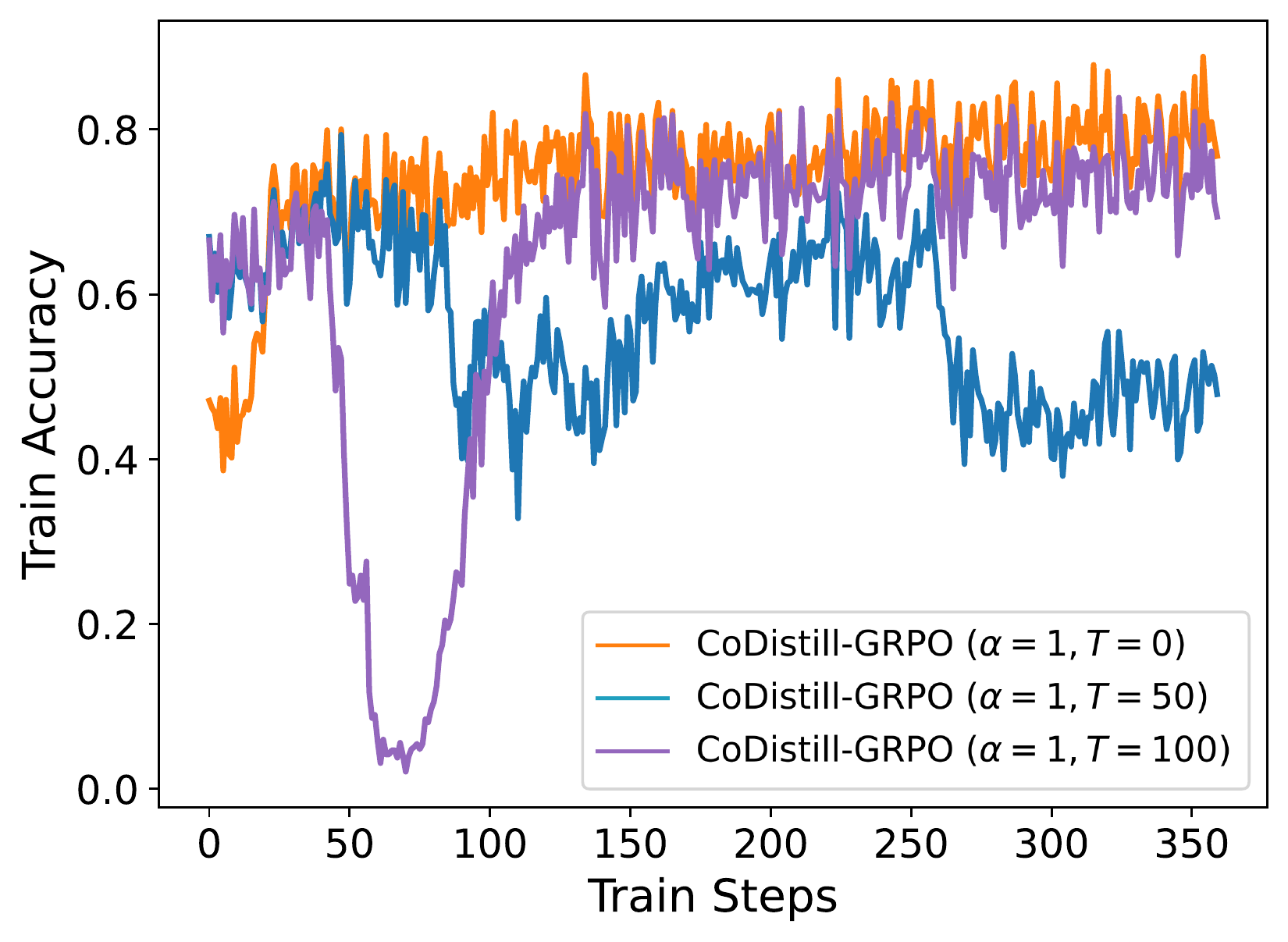}

    \end{subfigure}
    \hfill
    \begin{subfigure}{0.45\linewidth}
        \centering
        \includegraphics[width=\linewidth]{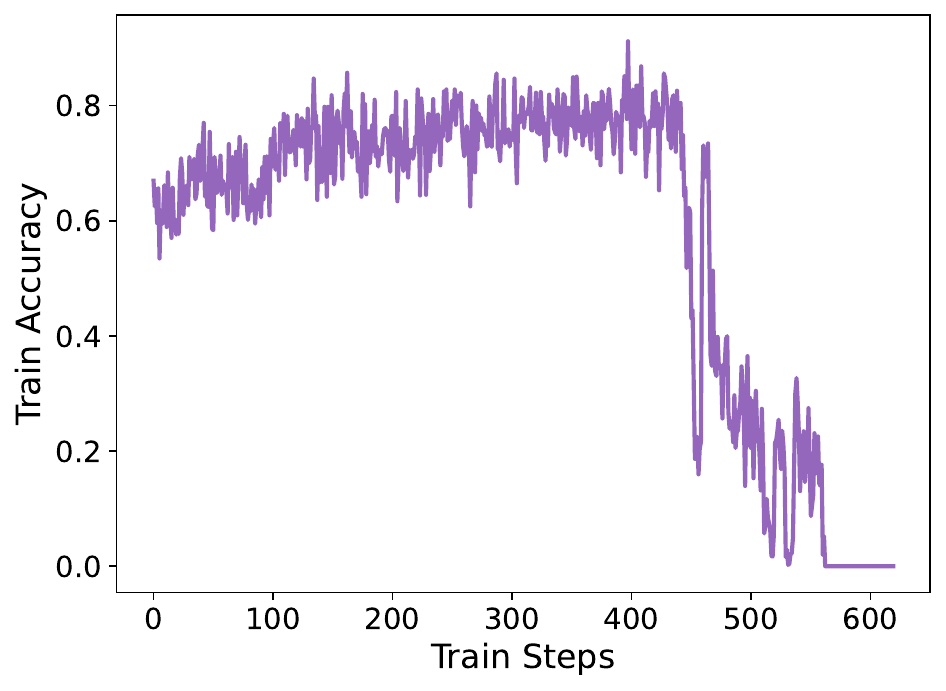}
    \end{subfigure}

    \caption{Left: Effects of using initial traces from the large model for all iterations of training for different values of $T$. Providing initial traces does not always help and makes training unstable. Right: Plot of the train accuracy of the small model for ignoring the initial traces in the importance sampling for CoDistill-GRPO as done in BREAD. Ignoring the initial traces makes training of CoDistill-GRPO unstable.}
    \label{fig:initial_trace_ign}
\end{figure}

\begin{figure}[h!]
    \centering
    \includegraphics[width=0.5\linewidth]{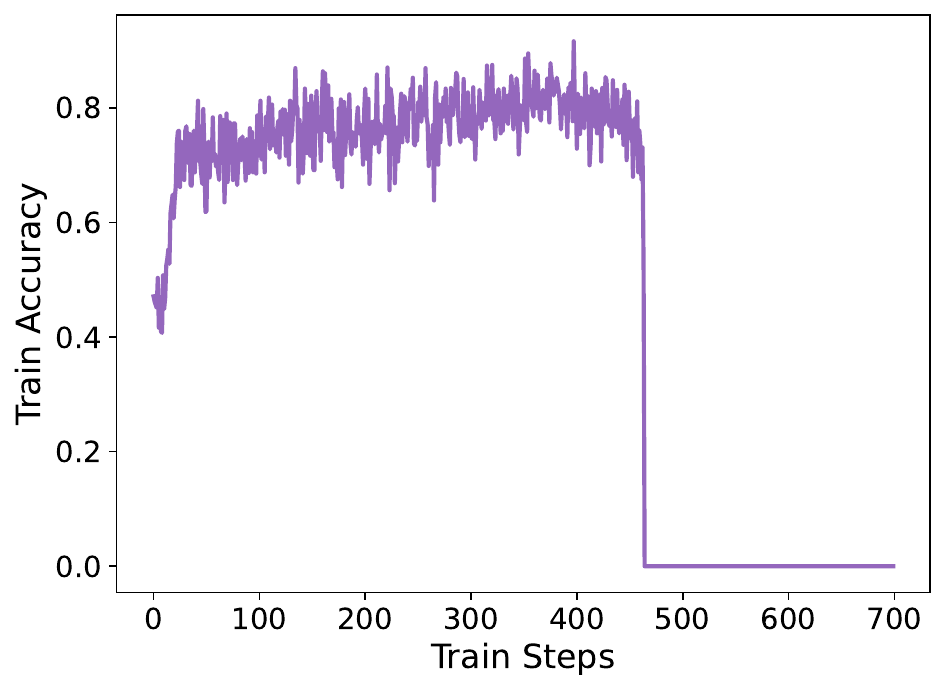}
    \caption{Plot of the train accuracy during training of the small model using the KL divergence as defined in \Cref{eqn:kl_div}. Using \Cref{eqn:kl_div} is less robust as opposed to using the KL divergence in the reward of \Cref{eqn:small_reward}.}
    \label{fig:skl_failure}
\end{figure}

\clearpage

\section{Additional Figures}
\label{sec:add_figures}

In this section, we present the test accuracy across training iterations that were deferred in \Cref{sec:experiments}.

\begin{figure}[h!]
    \centering
    \includegraphics[width=0.9\linewidth]{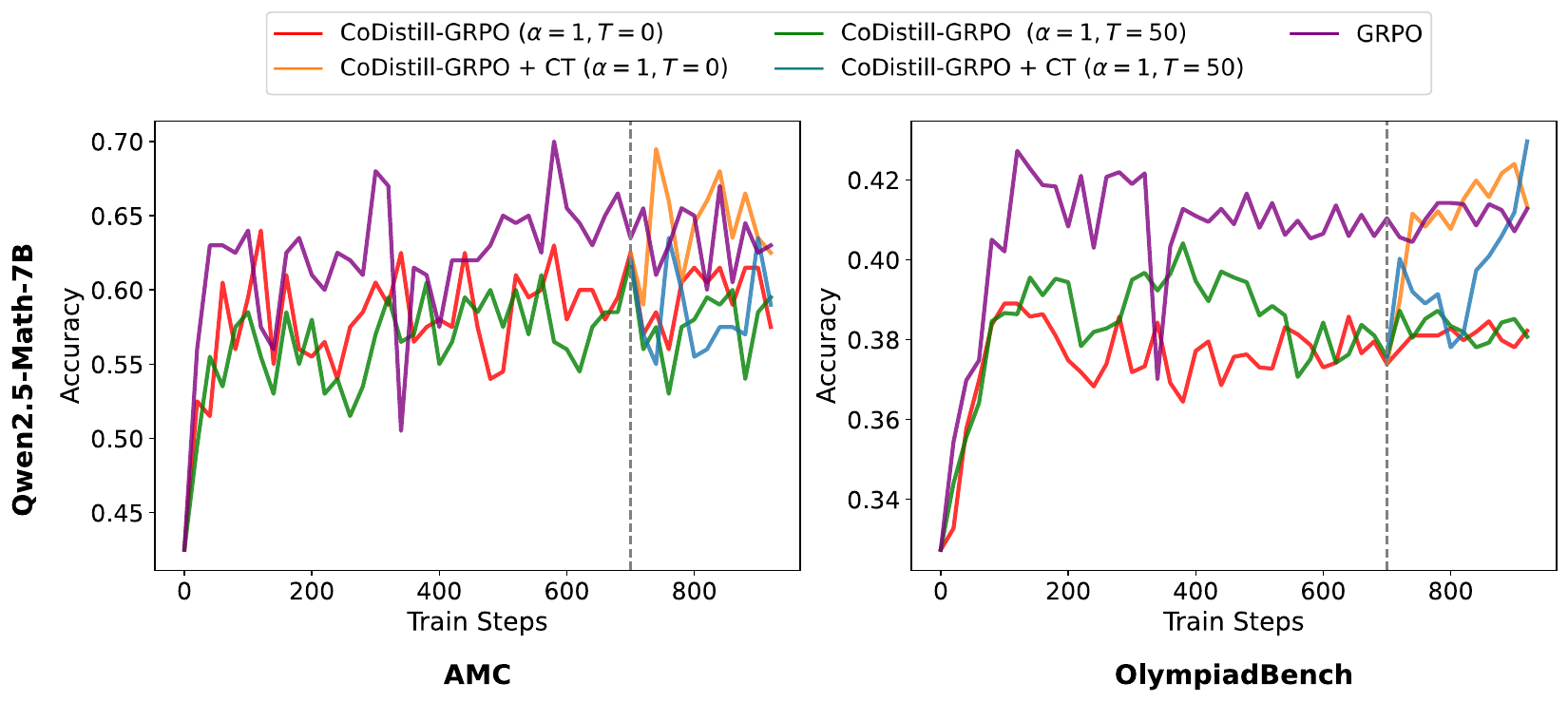}
    \caption{Test accuracy of the Qwen2.5-Math-7B model across training iterations on the AMC and OlympiadBench datasets.}
    \label{fig:7b_amc_ob}
\end{figure}

\begin{figure}[h!]
    \centering
    \includegraphics[width=0.9\linewidth]{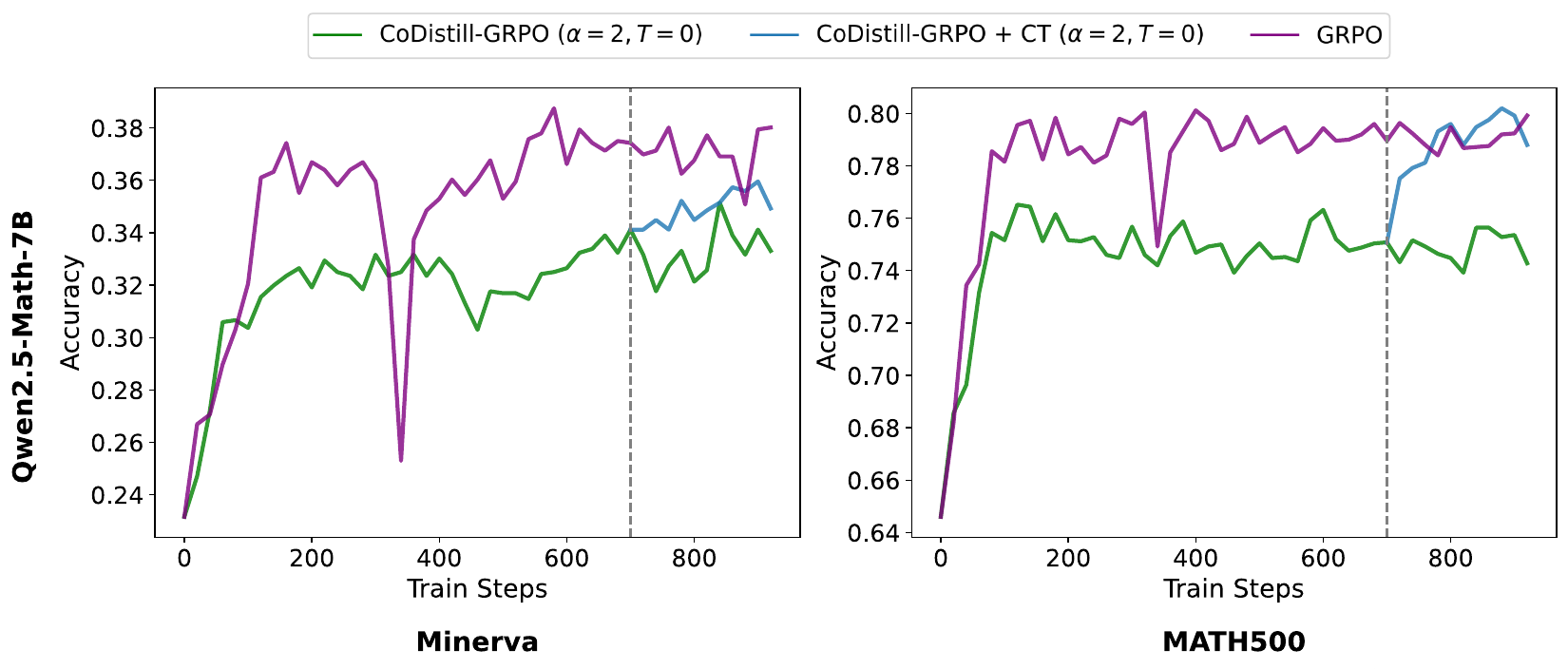}
    \caption{Test accuracy of the Qwen2.5-Math-7B model across training iterations on the Minerva and MATH500 datasets with $\alpha = 2, T=0$.}
    \label{fig:7b_minerva_math_a2}
\end{figure}

\begin{figure}[h!]
    \centering
    \includegraphics[width=0.9\linewidth]{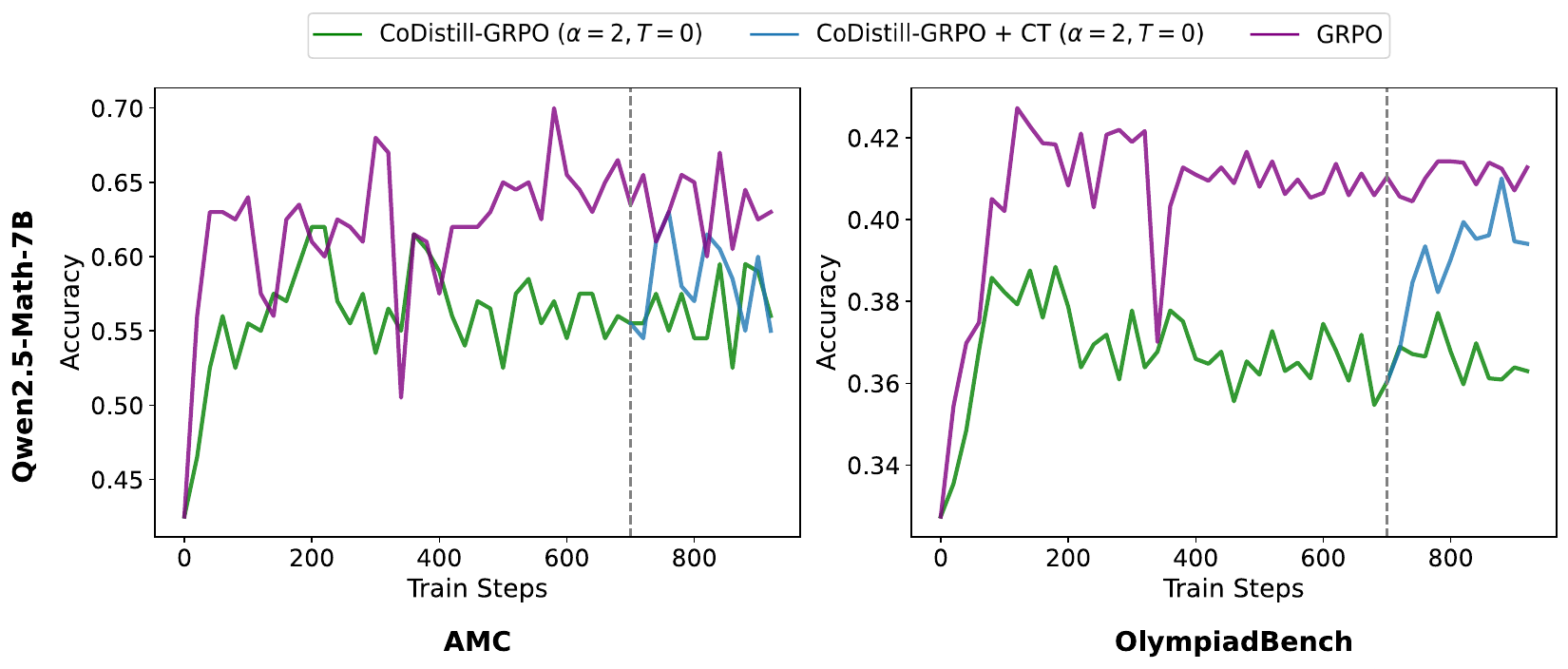}
    \caption{Test accuracy of the Qwen2.5-Math-7B model across training iterations on the AMC and OlympiadBench datasets with $\alpha = 2, T=0$.}
    \label{fig:7b_amc_ob_a2}
\end{figure}

\begin{figure}[h!]
    \centering
    \includegraphics[width=0.9\linewidth]{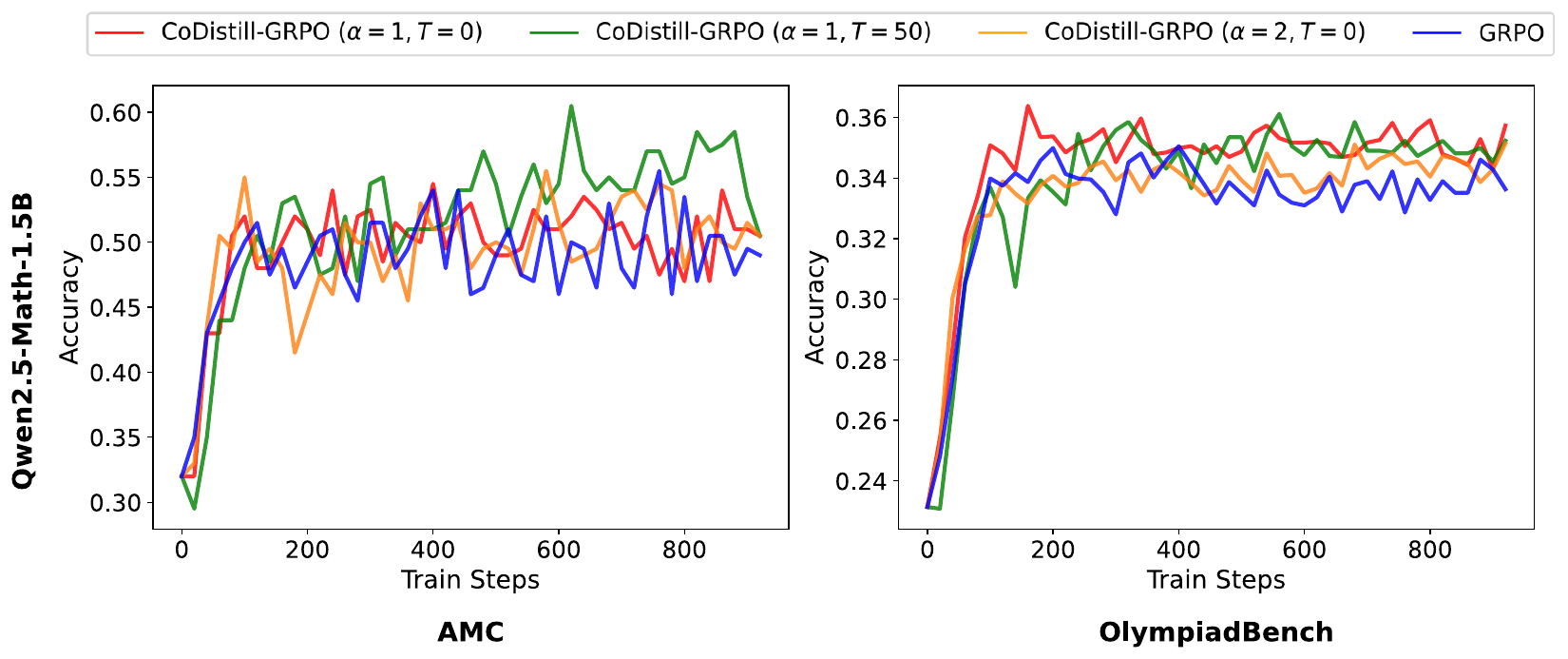}
    \caption{Test accuracy of the Qwen2.5-Math-1.5B model across training iterations on the AMC and OlympiadBench datasets.}
    \label{fig:1_5b_amc_ob}
\end{figure}

\clearpage

\section{Deferred Proofs}
\label{sec:deferred_proofs}

\subsection{Proof of \Cref{lem:unbiased}}

\begin{proof}
    Recall that the reward $\widetilde{r}_i$ depends on the model parameters $\phi$. Using the product rule, the gradient of the expected loss in \Cref{eqn:small_obj} with respect to $\phi$ yields
    \begin{align*}
        \nabla_{\phi} \mc{L}_{\texttt{Small}}(\phi) \, |_{\phi = \phi_{\text{old}}} &= \frac{1}{G}\sum_{i=1}^G  \sum_{t=1}^N \mbb{E} \biggl[ \underbrace{\widetilde{A}_{i, t} \cdot \nabla_{\phi} \log \pi_{\phi}\left( o_{i,t} \, | \, q, o_{i, <t}\right)}_{\text{policy gradient}} +  \underbrace{\nabla_{\phi}\widetilde{r}_i(\phi) - \nabla_{\phi} \mathrm{mean}\left(\widetilde{\mathbf{r}}(\phi) \right)}_{\text{reward gradient from KD}}\biggr],
    \end{align*}
where 
\begin{align*}
    \nabla_{\phi}\widetilde{r}_i(\phi) = -\alpha \cdot \frac{1}{N} \sum_{t=1}^N \nabla_{\phi} \log \pi_{\phi}\left( o_{i,t} \, | \, q, o_{i, <t}\right).
\end{align*}
Notice that 
\begin{align}
\label{eqn:kd_term_zero}
    \mbb{E}_{o_i \sim \pi_{\phi}(\cdot | q)}\left[ \nabla_{\phi}\widetilde{r}_i(\phi)\right] &= -\alpha \cdot \frac{1}{N} \sum_{t=1}^N \mbb{E}_{o_i \sim \pi_{\phi}(\cdot | q)}\left[\nabla_{\phi} \log \pi_{\phi}\left( o_{i,t} \, | \, q, o_{i, <t}\right)\right] \\
    &= -\alpha \cdot \frac{1}{N} \sum_{t=1}^N \mbb{E}_{o_{i, <t}}\left[\mbb{E}_{o_{i, t} \sim \pi_{\phi}(\cdot | q, o_{i, <t})}\left[\nabla_{\phi} \log \pi_{\phi}\left( o_{i,t} \, | \, q, o_{i, <t}\right)\right]\right] \tag{Tower Property}\\
    &= -\alpha \cdot \frac{1}{N} \sum_{t=1}^N \mbb{E}_{o_{i, <t}}\left[\nabla_{\phi}\int \pi_{\phi}\left( o_{i,t} \, | \, q, o_{i, <t}\right) \right] \tag{Leibniz Rule}\\
    &= -\alpha \cdot \frac{1}{N} \sum_{t=1}^N \nabla_{\phi}1 \\
    &= 0,
\end{align}
where the tower property was used for each token index $t$.
Hence, the reward gradient from KD is zero in expectation, yielding the following remaining gradient:
\begin{align*}
    \nabla_{\phi} \mc{L}_{\texttt{Small}}(\phi) \, |_{\phi = \phi_{\text{old}}} &= \frac{1}{G}\sum_{i=1}^G  \sum_{t=1}^N \mbb{E} \biggl[ \widetilde{A}_{i, t} \cdot \nabla_{\phi} \log \pi_{\phi}\left( o_{i,t} \, | \, q, o_{i, <t}\right) \biggr].
\end{align*}
For simplicity in notation, let us denote 
\begin{align*}
    s_i \coloneqq \sum_{t=1}^N \nabla_{\phi} \log \pi_{\phi}\left( o_{i,t} \, | \, q, o_{i, <t}\right).
\end{align*}
Then, using the fact that $\widetilde{A}_{i, t} = \widetilde{A}_i = \widetilde{r}_i - \mathrm{mean}(\widetilde{\mathbf{r}})$, notice that
\begin{align*}
     \frac{1}{G} \sum_{i=1}^G \widetilde{A}_i \cdot s_i &= \frac{1}{G}\sum_{i=1}^G  \widetilde{r}_i \cdot s_i -  \frac{1}{G} \mathrm{mean}(\widetilde{\mathbf{r}}) \sum_{i=1}^G s_i \\
     &= \frac{1}{G}\sum_{i=1}^G  \widetilde{r}_i \cdot s_i - \frac{1}{G^2} \left(\sum_{j=1}^G \widetilde{r}_j \right) \left(\sum_{i=1}^G s_i \right).
\end{align*}
By taking the conditional expectation on $q$, we obtain
\begin{align*}
    \mbb{E}\left[ \frac{1}{G}\sum_{i=1}^G  \widetilde{r}_i \cdot s_i \, | \, q \right] - \mbb{E}\left[ \frac{1}{G^2} \left(\sum_{j=1}^G \widetilde{r}_j \right) \left(\sum_{i=1}^G s_i \right) \, | \, q \right] &= \mbb{E}[\widetilde{r}_i s_i \, | \, q] - \frac{1}{G^2}\sum_i^{G} \mbb{E}[\widetilde{r}_i s_i \, | \, q] - \frac{1}{G^2}\underbrace{\sum_{i \neq j}\mbb{E}[\widetilde{r}_j s_i \, | \, q]}_{=0 \text{  (by independence)}} \\
    &= \left(\frac{G-1}{G}\right) \mbb{E}\left[\widetilde{r}_i s_i \, | \, q \right]
\end{align*}
Taking the average over $q$, we have
\begin{align}
\label{eqn:grpo_grad}
    \nabla_{\phi} \mc{L}_{\texttt{Small}}(\phi) = \left(\frac{G-1}{G}\right) \cdot \mbb{E}\left[\widetilde{r}_i \cdot  \sum_{t=1}^N \nabla_{\phi} \log \pi_{\phi}\left( o_{i,t} \, | \, q, o_{i, <t}\right)  \right].
\end{align}
Now we derive the true policy gradient by defining the reward objective:
\begin{align*}
    J(\phi) \coloneqq \mbb{E}_{q\sim P(Q), o_i\sim \pi_{\phi}(\cdot | q)} \left[ \widetilde{r}_{\phi}(q, o_i)\right],
\end{align*}
which yields
\begin{align*}
    \nabla_{\phi}J(\phi) &= \mbb{E} \left[ \widetilde{r}_{\phi}(q, o_i) \cdot \nabla_{\phi}\log \pi_{\phi}\left(o_i \, | \, q\right) \right] + \underbrace{\mbb{E}\left[\nabla_{\phi}\widetilde{r}_{\phi}(q, o_i) \right]}_{=0 \text{  (derived in \Cref{eqn:kd_term_zero})}} \\
    &= \mbb{E} \left[ \widetilde{r}_{\phi}(q, o_i) \cdot \sum_{t=1}^N\nabla_{\phi}\log \pi_{\phi}\left(o_{i, t} \, | \, q, o_{i, <t}\right) \right].
\end{align*}
Then by multiplying the learning rate $\eta$ on both sides for the update rule:
\begin{align*}
    \eta \cdot \nabla_{\phi}J(\phi) &= \eta \cdot \mbb{E}\left[\widetilde{r}_i \cdot \sum_{t=1}^N \nabla_{\phi} \log \pi_{\phi}\left( o_{i,t} \, | \, q, o_{i, <t}\right)  \right] \\
    &= \underbrace{\eta \left(\frac{G}{G-1} \right)}_{= \eta_{\text{eff}}} \nabla_{\phi} \mc{L}_{\texttt{Small}}(\phi) \\ 
    &= \eta_{\text{eff}} \cdot\nabla_{\phi} \mc{L}_{\texttt{Small}}(\phi),
\end{align*}
which completes the proof. We remark that the adjustment in learning rate also yields a gradient equivalent to the RLOO~\citep{ahmadian2024basicsrevisitingreinforcestyle} gradient, which suggests an alternative direction for the proof.
\end{proof}

\subsection{Proof of \Cref{thm:gradient}}

\begin{proof}
    Let us define
    \begin{align*}
        J_q(\phi) \coloneqq \mbb{E}_{o_i \sim \pi_{\phi}(\cdot | q)}\left[\widetilde{r}_{\phi}(q, o_i)\right] \quad \text{such that} \quad J(\phi) = \mbb{E}_{q \sim P(Q)}\left[J_q(\phi)\right].
    \end{align*}
    Recall that by \Cref{lem:unbiased}, we have 
    \begin{align*}
        \eta_{\text{eff}} \cdot \mbb{E}\left[\nabla_{\phi} \mc{L}_{\texttt{Small}}(\phi) \, | \, q\right] = \eta \cdot \nabla_{\phi} J_q(\phi),
    \end{align*}
    conditioned on $q$, before averaging over $q$. Hence, it  suffices to decompose $J_q(\phi)$. Expanding the effective reward from \Cref{eqn:small_reward}:
    \begin{align*}
        J_q(\phi) = \mbb{E}_{o_i \sim \pi_{\phi}(\cdot | q)}\left[r(q, o_i)\right] + \alpha \cdot \mbb{E}_{o_i \sim \pi_{\phi}(\cdot | q)}\left[\frac{1}{N}\sum_{t=1}^N \log \frac{\pi_{\theta}(o_{i,t} | q, o_{i, <t})}{\pi_{\phi}(o_{i,t} | q, o_{i, <t})}\right].
    \end{align*}
    By the chain rule, we have
    \begin{align*}
        \sum_{t=1}^N \log \frac{\pi_{\theta}(o_{i,t} | q, o_{i, <t})}{\pi_{\phi}(o_{i,t} | q, o_{i, <t})} = \log \frac{\pi_{\theta}(o_i | q)}{\pi_{\phi}(o_i | q)},
    \end{align*}
    and so the second term becomes
    \begin{align*}
        \frac{\alpha}{N} \cdot \mbb{E}_{o_i \sim \pi_{\phi}(\cdot | q)}\left[\log \frac{\pi_{\theta}(o_i | q)}{\pi_{\phi}(o_i | q)}\right] = -\frac{\alpha}{N} \sum_{o_i} \pi_{\phi}(o_i | q) \cdot \log \frac{\pi_{\phi}(o_i | q)}{\pi_{\theta}(o_i | q)} = -\frac{\alpha}{N} \cdot \mbb{D}_{\texttt{KL}}\left(\pi_{\phi}(\cdot | q) \| \pi_{\theta}(\cdot | q)\right).
    \end{align*}
    Substituting back, we obtain
    \begin{align*}
        J_q(\phi) = \mbb{E}_{o_i \sim \pi_{\phi}(\cdot | q)}\left[r(q, o_i)\right] - \frac{\alpha}{N} \mbb{D}_{\texttt{KL}}\left(\pi_{\phi}(\cdot | q) \| \pi_{\theta}(\cdot | q)\right).
    \end{align*}
    Both terms are differentiable in $\phi$, and so taking the gradient and applying \Cref{lem:unbiased}, we have
    \begin{align*}
        \eta_{\text{eff}} \cdot \mbb{E}\left[\nabla_{\phi} \mc{L}_{\texttt{Small}}(\phi) \;\big|\; q\right] = \eta \cdot \nabla_{\phi} \mbb{E}_{o_i \sim \pi_{\phi}(\cdot | q)}\left[r(q, o_i)\right] - \eta \cdot \frac{\alpha}{N} \nabla_{\phi} \mbb{D}_{\texttt{KL}}\left(\pi_{\phi}(\cdot | q) \| \pi_{\theta}(\cdot | q)\right),
    \end{align*}
    which completes the proof.
\end{proof}

\end{document}